\documentclass{article}

\usepackage[table]{xcolor} 
    \usepackage[final]{neurips_2025}

\usepackage{hyperref}
\usepackage{url}

\usepackage{amsmath}
\usepackage{amssymb}
\usepackage{mathtools}
\usepackage{amsthm}

\usepackage{microtype}
\usepackage{graphicx}
\usepackage{booktabs} 

\usepackage{algorithm}
\usepackage{algpseudocode}
\usepackage{graphicx} 
\usepackage{todonotes}

\usepackage{xcolor}
\usepackage{dsfont}
\usepackage{nicefrac}
\usepackage{tcolorbox}
\usepackage{tabularx}
\usepackage{pifont}
\usepackage[inline]{enumitem}
\usepackage{multirow}
\usepackage{caption}
\usepackage{wrapfig}
\usepackage[export]{adjustbox}
\usepackage[list=true]{subcaption}
\usepackage{colortbl}
\definecolor{lightgray}{RGB}{242, 242, 242}
\usepackage{float}
\usepackage[capitalize,noabbrev]{cleveref}
\crefname{assumption}{Assumption}{Assumptions}
\crefname{algorithm}{Algorithm}{Algorithms}

\usepackage{longtable}
\usepackage[title]{appendix}

\usepackage{tcolorbox}
\usepackage{xspace}
\DeclareRobustCommand{\ie}{i.e.,\@\xspace}
  
\DeclareRobustCommand{\eg}{e.g.,\@\xspace}
\DeclareRobustCommand{\wrt}{w.r.t.\@\xspace}
\usepackage[capitalize]{cleveref}
\usepackage{tikz}
\usepackage{tikz-cd}
\usetikzlibrary{shapes,backgrounds}
\usepackage{array}
\usepackage{enumitem}

\usepackage{wrapfig}
\usepackage[export]{adjustbox}
\usepackage[list=true]{subcaption}
\usepackage[normalem]{ulem}
\usepackage{rotating}   
\usepackage{cleveref}
\usepackage{adjustbox}
\usepackage{twoopt}

\usepackage[T1]{fontenc}
\usepackage[utf8]{inputenc}
\usepackage{amsmath,amssymb}
\usepackage{booktabs}      
\usepackage{multirow}        

\usepackage{acronym}		

\usepackage{amsthm}
\usepackage{thmtools, thm-restate}
\declaretheorem[numberwithin=section]{thm}
\declaretheorem[sibling=thm]{theorem}

\declaretheorem[numberwithin=section]{assumption}
\declaretheorem[]{proof sketch}
\declaretheorem[]{definition}

\usepackage{amsfonts}[mathscr]
\usepackage{amssymb}
\usepackage{amsmath}
\DeclareMathOperator*{\EV}{\mathbb{E}}

\newcommand{\A}{\mathcal{A}}

\DeclareMathOperator*{\argmax}{arg\,max}

\newcommand{\X}{\mathcal{X}}
\newcommand{\R}{\mathbb{R}}

\newcommand{\F}{\mathcal{F}}

\newcommand{\D}{\mathcal{D}}

\newcommand{\Q}{\mathcal{Q}}
\newcommand{\G}{\mathcal{G}}
\newcommand{\mypar}[1]{\textbf{#1.}}

\newcommand{\entropy}{\mathcal{H}}
\DeclareMathOperator{\mP}{\mathbb{P}}

\DeclareMathOperator{\mF}{\mathbb{F}}
\newcommand{\der}{\mathrm{d}}
\newcommand{\divergence}{\mathrm{div}}

\newcommand{\LinearFineTuningSolver}{\textsc{\small{EntropyRegularizedControlSolver}}\xspace}
\newcommand{\noise}{U}
\newcommand{\bias}{b}
\newcommand{\hist}{\mathcal{T}}
\newcommand{\pMD}{p_\sharp}
\newcommand{\step}{\gamma}

\newcommand{\norm}[1]{\left\| #1 \right\|}

\usepackage{pifont}
\newcommand{\tick}{\ding{51}}  
\newcommand{\cross}{\ding{55}} 
\usepackage{makecell}
\definecolor{darkgreen}{RGB}{0,100,0}
\definecolor{darkorange}{RGB}{255,140,0}

\newcommand{\AlgNameLong}{Flow Density Control\xspace}
\newcommand{\AlgNameShort}{\textsc{\small{FDC}}\xspace}
\newcommand{\AlgNameDef}{\textbf{F}low \textbf{D}ensity \textbf{C}ontrol  (\textsc{\small{FDC}}\xspace)}

\newcommand{\AlgNameShortAM}{\textsc{\small{AM}}\xspace}

\definecolor{myviolet}{rgb}{0.6, 0.4, 0.8}
\definecolor{mygreen}{rgb}{0.0, 0.5, 0.0}

\newcommand{\debug}[1]{#1}

\newcommand{\newmacro}[2]{\newcommand{#1}{{#2}}}		

\newcommand{\dual}{h}
\newcommand{\run}{k}

\newcommand{\state}{\dual}
\newcommand{\curr}[1][\state]{\debug{#1}^{\run}}		

\newcommand{\efftime}{\tau}
\newcommand{\apt}[2][]{\state^{#1}(#2)}	
\newcommand{\ctime}{t}	
\newcommand{\defeq}{\coloneqq}

\newmacro{\temp}{\eta}		
\newmacro{\points}{\mathcal{Z}}		
\newmacro{\intpoints}{\points^{\circ}}		
\newmacro{\point}{\dual}		
\newmacro{\pointalt}{\alt\point}		

\newmacro{\ctimealt}{s}		
\newmacro{\cstart}{0}		

\newmacro{\horizon}{T}		

\newmacro{\vecfield}{V}		

\newmacro{\signal}{V}		
\newmacro{\error}{W}		
\newmacro{\brown}{W}		

\newmacro{\dstate}{Y}		

\newmacro{\flowmap}{\Theta}		
\newcommandtwoopt{\flow}[2][\ctime][\point]{\flowmap_{#1}(#2)}
\newmacro{\minmax}{\Phi}		

\newmacro{\minvar}{x}		
\newmacro{\minvaralt}{\alt x}		
\newmacro{\minvars}{\mathcal{X}}		

\newmacro{\maxvar}{y}		
\newmacro{\maxvaralt}{\alt y}		
\newmacro{\maxvars}{\mathcal{Y}}		

\newmacro{\minsol}{\sol[\minvar]}		
\newmacro{\maxsol}{\sol[\maxvar]}		

\newcommand{\sol}[1][\point]{#1^{\ast}}		

\newmacro{\set}{\mathcal{S}}		

\newmacro{\open}{\mathcal{U}}		
\newmacro{\closed}{\mathcal{C}}		
\newmacro{\cpt}{\mathcal{K}}		
\newmacro{\nhd}{\mathcal{U}}		

\newcommand{\afterhead}{.}		
\newcommand{\para}[1]{\paragraph{\textbf{#1\afterhead}}}

\newacro{APT}{asymptotic pseudotrajectory}
\newacroplural{APT}[APTs]{asymptotic pseudotrajectories}
\newacro{GD}{gradient dynamics}
\newacro{GF}{gradient flow}
\newacro{ICT}{internally chain-transitive}
\newacro{MDS}{martingale difference sequence}
\newacro{NE}{Nash equilibrium}
\newacroplural{NE}[NE]{Nash equilibria}
\newacro{ODE}{ordinary differential equation}
\newacro{SA}{stochastic approximation}
\newacro{SFO}{stochastic first-order oracle}
\newacro{SG}{stochastic gradient}
\newacro{SP}{saddle-point}
\newacro{WAC}{weak asymptotic coercivity}

\newacro{AH}{Arrow\textendash Hurwicz}
\newacro{BDG}{Burkholder\textendash Davis\textendash Gundy}
\newacro{ConO}{consensus optimization}
\newacro{RM}{Robbins\textendash Monro}
\newacro{KW}{Kiefer\textendash Wolfowitz}
\newacro{GDA}{gradient descent/ascent}
\newacro{SGA}{symplectic gradient adjustment}
\newacro{SGD}{stochastic gradient descent}
\newacro{SGDA}{stochastic gradient descent/ascent}
\newacro{SPSA}{simultaneous perturbation stochastic approximation}
\newacro{ASGDA}[alt-SGDA]{alternating stochastic gradient descent/ascent}
\newacro{SEG}{stochastic extra-gradient}
\newacro{EG}{extra-gradient}
\newacro{PEG}{Popov's extra-gradient}
\newacro{RG}{reflected gradient}
\newacro{OG}{optimistic gradient}
\newacro{PPM}{proximal point method}

\newacro{GAN}{generative adversarial network}
\newacro{NN}{neural network}
\newacro{FTRL}{``follow the regularized leader''}
\newacro{CGD}{Competitive Gradient Descent}
\newacro{wp1}[w.p.$1$]{with probability $1$}

\definecolor{pastelblueold}{RGB}{56,146,236}
\definecolor{pastelblue}{RGB}{43,115,187}
\definecolor{pastelgreen}{RGB}{63,159,95}

\hypersetup{
  colorlinks=true,
  citecolor=pastelgreen,
  linkcolor=pastelblue,    
  urlcolor=pastelblue      
}

\title{Flow Density Control: Generative Optimization Beyond Entropy-Regularized Fine-Tuning}

\author{%
  Riccardo De Santi \\
  ETH Zurich \\
  ETH AI Center\\
  \texttt{rdesanti@ethz.ch} \\
  \And
   Marin Vlastelica \\
   ETH Zurich \\
   ETH AI Center\\
   \texttt{marin.vlastelica@inf.ethz.ch} \\
  \And
   Ya-Ping Hsieh\\
   ETH Zurich \\
   \texttt{yaping.hsieh@inf.ethz.ch} \\
   \And
   Zebang Shen \\
   ETH Zurich \\
   \texttt{zebang.shen@inf.ethz.ch} \\
   \And
   Niao He \\
   ETH Zurich \\
   ETH AI Center\\
   \texttt{niaohe@ethz.ch} \\
   \And
   Andreas Krause \\
   ETH Zurich \\
   ETH AI Center\\
   \texttt{krausea@ethz.ch} \\
}

\begin{document}

\maketitle

\begin{abstract}
\looseness -1 Adapting large-scale foundation flow and diffusion generative models to optimize task-specific objectives while preserving prior information is crucial for real-world applications such as molecular design, protein docking, and creative image generation. 
Existing principled fine-tuning methods aim to maximize the expected reward of generated samples, while retaining knowledge from the pre-trained model via KL-divergence regularization. 
In this work, we tackle the significantly more general problem of optimizing general utilities beyond average rewards, including risk-averse and novelty-seeking reward maximization, diversity measures for exploration, and experiment design objectives among others. Likewise, we consider more general ways to preserve prior information beyond KL-divergence, such as optimal transport distances and Rényi divergences. 
To this end, we introduce \AlgNameDef, a simple algorithm that reduces this complex problem to a specific sequence of simpler fine-tuning tasks, each solvable via scalable established methods. We derive convergence guarantees for the proposed scheme under realistic assumptions by leveraging recent understanding of mirror flows. Finally, we validate our method on illustrative settings, text-to-image, and molecular design tasks, showing that it can steer pre-trained generative models to optimize objectives and solve practically relevant tasks beyond the reach of current fine-tuning schemes.

\end{abstract}

\addtocontents{toc}{\protect\setcounter{tocdepth}{-1}}
\vspace{-2mm}
\section{Introduction} \vspace{-2mm}
\label{sec:introduction}
\begin{wrapfigure}{r}{0.45\textwidth}
  \centering \vspace{-7mm}
  \includegraphics[width=0.4\textwidth]{images/table_img.pdf}
  \caption{\looseness -1 We extend the capabilities of current fine‐tuning schemes from KL‐regularized expected reward maximization (left) to the optimization of arbitrary distributional utilities $\F$ under general divergences $\D$ (right).
} \vspace{-5mm}
  \label{fig:table_F}
\end{wrapfigure}

\looseness -1 Large-scale generative modeling has recently seen remarkable advancements, with flow~\citep{lipman2022flow, lipman2024flow} and diffusion models~\citep{sohl2015deep, song2019generative, ho2020denoising} standing out for their ability to produce high-fidelity samples across a wide range of applications, from chemistry~\citep{hoogeboom2022equivariant} and biology~\citep{corso2022diffdock} to robotics~\citep{chi2023diffusion}. 
However, approximating the data distribution is insufficient for real-world applications such as scientific discovery~\citep{bilodeau2022generative, zeni2023mattergen}, where one typically wishes to generate samples optimizing specific utilities, \eg molecular stability and diversity, while preserving certain information from a pre-trained model. This problem has recently been tackled via fine-tuning in the case where the utility corresponds to the expected reward of generated samples, and pre-trained model information is retained via KL-divergence regularization, as shown in Fig. \ref{fig:table_F} (left). Crucially, this specific fine-tuning problem can be solved via entropy-regularized control formulations~\citep[\eg][]{domingo2024adjoint, uehara2024fine, tang2024fine} with successful applications in real-world domains such as image generation~\citep{domingo2024adjoint}, molecular design~\citep{uehara2024feedback}, or protein engineering~\citep{uehara2024feedback}. 

\looseness -1 Unfortunately, many practically relevant tasks cannot be captured by this formulation. For instance, consider the tasks of \emph{risk‑averse} and \emph{novelty‑seeking} reward maximization. In the former case, one wishes to steer the generative model toward distributions with controlled worst-case rewards, thereby improving validity and safety. In the latter case, one aims to control the upper tail of the reward distribution to maximize the probability of generating exceptionally promising designs, \eg for scientific discovery.
Other applications that cannot be captured via maximization of simple expectations include manifold exploration~\citep{de2025provable}, model de-biasing~\citep{decruyenaere2024debiasing}, and optimal experimental design~\citep{mutny2023active, de2024geometric} among others. 
Similarly, preserving prior information via a KL divergence has known drawbacks. For instance, it can lead to missing of low-probability yet valuable modes~\citep{li2016renyi, pandey2024heavy}, and it prevents from leveraging the geometry of the space even when this is known, \eg in protein docking~\citep{corso2022diffdock}. Replacing KL with alternative divergences can address these shortcomings.
Driven by these motivations, in this work we aim to answer the following fundamental question (see Fig. \ref{fig:table_F}):
\vspace{-2mm}
\begin{center} 
\emph{How can we provably fine-tune a flow or diffusion model to optimize any user-specified utility \\ while preserving prior information via an arbitrary divergence?} \vspace{-1mm}
\end{center}
\looseness -1 Answering this would contribute to the algorithmic-theoretical foundations of \emph{generative optimization}. \vspace{-7mm}
\looseness -1 \paragraph{{Our approach}} We tackle this challenge by first introducing the formal problem of \emph{generative optimization via fine-tuning}.
Then, we shed light on why this formulation is strictly more expressive than current fine-tuning problems~\citep{domingo2024adjoint, tang2024fine}, and present a sample of novel practically relevant utilities and divergences (Sec. \ref{sec:problem_setting}). Next, we introduce \AlgNameDef, a simple sequential scheme that can fine-tune models to optimize general objectives beyond the reach of entropy-regularized control methods. This is achieved by leveraging recent machinery from Convex~\citep{hazan2019maxent} and General Utilities RL~\citep{zhang2020variational} (Sec. \ref{sec:algorithm}). We provide rigorous convergence guarantees for the proposed algorithm in both a simplified scenario, via convex optimization analysis~\citep{nemirovskij1983problem, lu2018relatively}, and in a realistic setting, by building on recent understanding of mirror flows~\citep{hsieh2019finding} (Sec. \ref{sec:theory}). Finally, we provide an experimental evaluation of the proposed method, demonstrating its practical relevance on both synthetic and high-dimensional image and molecular generation tasks, showing how it can steer pre-trained models to solve tasks beyond the inherent limits of current fine-tuning schemes (Sec. \ref{sec:experiments}).
\vspace{-3mm}
\paragraph{{Our contributions}} To sum up, in this work we contribute
\begin{itemize}[noitemsep,topsep=0pt,parsep=0pt,partopsep=0pt,leftmargin=*]
    \item A formalization of the {\em generative optimization} problem, which extends current fine-tuning formulations beyond linear utilities and general divergences  (Sec. \ref{sec:problem_setting}).
    \item {\em \AlgNameDef}, a principled algorithm capable of optimizing functionals beyond the reach of current fine-tuning schemes based on entropy-regularized control/RL (Sec. \ref{sec:algorithm}).
    \item Convergence guarantees for the presented algorithm both under simplified and realistic assumptions leveraging recent understanding of mirror flows (Sec. \ref{sec:theory}).
    \item An experimental evaluation of \AlgNameShort showcasing its practical relevance on both illustrative and high-dimensional text-to-image and molecular design tasks, showing how it can steer pre-trained models to solve tasks beyond the capabilities of current fine-tuning schemes. (Sec. \ref{sec:experiments}).
\end{itemize}

\vspace{-2mm}
\section{Background and Notation} \vspace{-2mm}
\label{sec:background}
\mypar{General Notation}
We denote with $\X \subseteq \R^d$ an arbitrary set. Then, we indicate the set of Borel probability measures on $\X$ with $\mP(\X)$, and the set of functionals over the set of probability measures $\mP(\X)$ as $\mF(\X)$. Given an integer $N$, we define $[N] \coloneqq \{1, \ldots, N\}$. 

\mypar{Generative Flow Models}
Generative models aim to approximately sample novel data points from a data distribution $p_{data}$. Flow models tackle this problem by transforming samples $X_0 = x_0$ from a source distribution $p_0$ into samples $X_1=x_1$ from the target distribution $p_{data}$\cite{lipman2024flow, farebrother2025temporal}. Formally, a \emph{flow} is a time-dependent map $\psi: [0,1]\times \R^d \to \R$ such that $\psi: (t,x) \to \psi_t(x)$. A \emph{generative flow model} is a continuous-time Markov process $\{X_t\}_{0 \leq t \leq 1}$ obtained by applying a flow $\psi_t$ to $X_0 \sim p_0$ as $X_t = \psi_t(X_0)$, $ t \in [0,1]$, such that $X_1 = \psi_1(X_0) \sim p_{data}$. In particular, the flow $\psi$ can be defined by a \emph{velocity field} $u: [0,1] \times \R^d \to \R^d$, which is a vector field related to $\psi$ via the following ordinary differential equation (ODE), typically referred to as \emph{flow ODE}:\vspace{-1.5mm}
\begin{equation}
    \frac{\der}{\der t} \psi_t(x) = u_t(\psi_t(x)) \label{eq:flow_diff_eq} 
\end{equation}
with initial condition $\psi_0(x) = 0$. A flow model $X_t = \psi_t(X_0)$ induces a probability path of \emph{marginal densities} $p = \{p_t\}_{0 \leq t \leq 1}$ such that at time $t$ we have that $X_t \sim p_t$. Given a velocity field $u$ and marginal densities $p$, we say that $u$ generates the marginal densities $p = \{p_t\}_{0 \leq t \leq 1}$ if $X_t = \psi_t(X_0) \sim p_t$ for all $t \in [0,1)$. This is the case if the pair $(u, p)$ satisfy the \emph{Continuity Equation}: \vspace{-3mm}
\begin{equation}
    \frac{\der}{\der t}p_t(x) + \divergence(p_t u_t)(x) = 0 \label{eq:continuity_equation}
\end{equation}
\looseness -1 In this case, we denote by $p^u$ the probability path of marginal densities induced by the velocity field $u$. Flow matching~\cite{lipman2022flow, liu2022flow, albergo2022building, lipman2024flow} can estimate a velocity field $u^{\theta}$ s.t. the induced marginal densities $p^{u_\theta}$ satisfy $p^{u_\theta}_0 = p_0$ and $p^{u_\theta}_1 = p_{data}$, where $p_0$ denotes the source distribution, and $p_{data}$ the target data distribution. Interestingly, diffusion models~\citep{song2019generative} (DMs) admit an equivalent ODE‐based formulation with identical marginal densities to their original SDE dynamics~\citep[Chapter 10]{lipman2024flow}. Consequently, although in this work we adopt the notation of flow models, our contributions carry over directly to DMs.

\mypar{Continuous-time Reinforcement Learning}
We formulate finite-horizon continuous-time reinforcement learning (RL) as a specific class of optimal control problems \citep{wang2020reinforcement, jia2022policy, treven2023efficient, zhao2024scores}. Given a state space $\X$ and an action space $\A$, we consider the transition dynamics governed by the following ODE:\vspace{-1.5mm}
\begin{equation} 
      \frac{\der}{\der t} \psi_t(x) = a_t(\psi_t(x)) \label{eqn_continuous_RL} 
\end{equation}
\looseness -1 where $a_t \in \A$ is a selected action. We consider a state space $\X \coloneqq \R^d \times [0,1]$, and denote by (Markovian) deterministic policy a function $\pi_t(X_t) \coloneqq \pi(X_t, t) \in \A$ mapping a state $(x,t) \in \X$ to an action $a \in \A$ such that $a_t = \pi(X_t, t)$, and denote with $p_t^\pi$ the marginal density at time $t$ induced by policy $\pi$. 

\mypar{Pre-trained Flow Models as an RL policy}
\looseness -1 A pre-trained flow model with velocity field $u^{pre}$ can be interpreted as an action process $a^{pre}_t \coloneqq u^{pre}(X_t, t)$, where $a_t^{pre}$ is determined by a continuous-time RL policy via $a_t^{pre} = \pi^{pre}(X_t, t)$~\citep{de2025provable}. Therefore, we can express the flow ODE induced by a pre-trained flow model by replacing $a_t$ with $a^{pre}$ in Eq. \eqref{eqn_continuous_RL}, and denote the pre-trained model by its (implicit) policy $\pi^{pre}$, which induces a marginal density $p^{pre}_1 \coloneqq p^{\pi^{pre}}_1$ approximating $p_{data}$.

\vspace{-2mm}
\section{Formal Problem: a General Framework for Generative Optimization} \vspace{-2mm}
\label{sec:problem_setting}
\begin{figure*}[t]
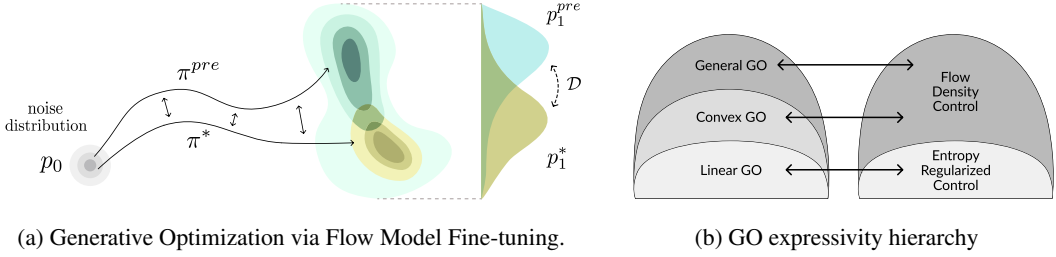

    \centering
    \begin{subfigure}{0.55\textwidth} 
        \centering
        \includegraphics[width=\textwidth, keepaspectratio]{images/gen_opt.pdf}
        \caption{Generative Optimization via Flow Model Fine-tuning.}
        \label{fig:gen_opt_drawing}
    \end{subfigure}%
    \hspace{15pt}
    \begin{subfigure}{0.4\textwidth} 
        \centering
        \raisebox{1ex}[0pt][0pt]{
        \includegraphics[width=\textwidth, keepaspectratio]{images/control_hierarchy.pdf}
        }
        \caption{GO expressivity hierarchy}
        \label{fig:control_hierarchy}
    \end{subfigure}\vspace{-0pt}
    \caption{\looseness -1 (\ref{fig:gen_opt_drawing}) Pre-trained and fine-tuned policies inducing densities $p^{pre}_1$ and optimal density $p_1^*$ \wrt utility $\F$ and divergence $\D$.  (\ref{fig:control_hierarchy}) Expressivity and control hierarchy for generative optimization. }
    \label{fig:joint_top_figures} \vspace{-0pt}
\end{figure*}

\looseness -1 In this section, we aim to formally introduce the general problem of generative optimization (GO) via fine-tuning. Formally, we wish to adapt a pre-trained generative flow model $\pi^{pre}$ to obtain a new model $\pi^*$ inducing an ODE:\vspace{-1mm}
\begin{equation} 
      \frac{\der}{\der t} \psi_t(x) = a^*_t(\psi_t(x)) \quad  \text{with} \quad a_t^* = \pi^*(x,t), \label{eq:controlled_ODE}
\end{equation}
\looseness -1 such that instead of imitating the data distribution $p_{data}$, as typically in generative modeling, it induces a marginal density $p_1^{\pi^*}$ that maximizes a utility measure $\F: \mP(\X) \to \R$, while preserving information from the pre-trained model $\pi^{pre}$ via regularization with an arbitrary divergence $\D( \cdot \, \| \, p^{pre})$. This algorithmic problem is illustrated in Fig. \ref{fig:gen_opt_drawing}, and formalized in the following.
\vspace{-1.3mm}
\begin{tcolorbox}[colframe=white!, top=2pt,left=2pt,right=2pt,bottom=2pt]
\begin{center}
\textbf{Generative Optimization via Flow Model Fine-Tuning}
\begin{align}
    &\argmax_{\pi}\quad  \F \left( p_1^{\pi} \right) - \alpha \D( p_1^{\pi}  \, \| \,  p_1^{pre})  \text{ s.t.} \frac{\der}{\der t}p_t(x) + \divergence(p_t a_t)(x) = 0 \text{ with } a_t = \pi(x,t)  \label{eq:gen_opt_problem}
\end{align}
\end{center}
\end{tcolorbox}
\vspace{-1.3mm}
\looseness -1 In this formulation, $\F$ and $\D$ are both functionals mapping the marginal density $p^\pi_1$ induced by policy $\pi$ to a scalar real number, namely $\F, \D: \mP(\X) \to \R$. The constraint in Eq. \eqref{eq:gen_opt_problem} is the (\emph{controlled}) Continuity Equation (see Eq. \eqref{eq:continuity_equation}), which relates the control policy $\pi$ to the induced marginal density $p_1^\pi$.
\vspace{-3mm}
\subsection{The sub-case of KL-regularized reward maximization via entropy-regularized control} \vspace{-2.5mm}
\looseness -1 Current fine-tuning schemes for flow generative models based on RL and control-theoretic formulations ~\citep[\eg][]{domingo2024adjoint, uehara2024fine} aim to tackle the following problem, where we omit the flow constraint for clarity:
\vspace{-1.3mm}
\begin{tcolorbox}[colframe=white!, top=2pt,left=2pt,right=2pt,bottom=2pt]
\begin{center}
\textbf{Linear Generative Optimization via Flow Model Fine-Tuning}
\begin{equation}
    \argmax_{\pi}\quad  \EV_{x \sim p_1^\pi} [r(x)] - \alpha D_{KL}( p_1^{\pi}  \, \| \,  p_1^{pre}) \label{eq:linear_fine_tuning}
\end{equation}
\end{center}
\end{tcolorbox}
\vspace{-1.3mm}
\looseness -1 Crucially, the common problem in Eq. \eqref{eq:linear_fine_tuning}, which we denote by \emph{Linear}\footnote{For clarity, we adopt the term \emph{linear} motivated by the linear utility even though the KL is non-linear.} GO, is the specific sub-case of the generative optimization problem in Eq. \eqref{eq:gen_opt_problem}, where the utility $\F$ is a linear functional corresponding to the expectation of a (reward) function $r: \X \to \R$, and $\D$ is the Kullback–Leibler divergence:
\begingroup
\setlength\abovedisplayskip{4pt}   
\setlength\belowdisplayskip{1pt}   
\begin{equation}
    \F(p_1^\pi) = \langle p_1^\pi, r \rangle = \EV_{x \sim p_1^\pi} [r(x)] \quad \text{ and } \quad \D(p_1^{\pi} \,\|\, p_1^{pre}) = D_{KL}( p_1^{\pi}  \, \| \,  p_1^{pre})
\end{equation}
\endgroup
\looseness -1 This specific fine-tuning problem can be solved via entropy-regularized (or relaxed) control~\citep{domingo2024adjoint}. 

\definecolor{lightblue}{RGB}{222,235,248}   
\definecolor{lightorange}{RGB}{255,237,217}
\colorlet{lightblueII}{lightblue!50!white}   
\colorlet{lightblueIII}{lightblue!75!white}  
\colorlet{lightorangeII}{lightorange!50!white}   
\colorlet{lightorangeIII}{lightorange!75!white}  

\begin{table*}[t]
\setlength{\tabcolsep}{4pt}
\renewcommand{\arraystretch}{2.0}
\begingroup
\everymath{\displaystyle}
\vspace{-1mm}
\begin{sc}
\resizebox{\textwidth}{!}{%
\begin{tabular}{c c c c c}
\toprule
\multirow{2}{*}{Application} &
\multirow{2}{*}{Functional $\F$ / $\D$} &
\multirow{2}{*}{\shortstack[c]{Linear GO}} &
\multicolumn{2}{c}{Non-Linear GO} \\ \cmidrule(lr){4-5}
 & & & Convex & General \\
\midrule
\rowcolor{lightblueII} Reward optimization~\citep{domingo2024adjoint, uehara2024fine} &
  $\mathbb{E}_{x \sim p^\pi}\left[r(x)\right]$ &
  \tick & \tick & \tick \\
\rowcolor{lightblue}%
\multirow{2}{*}{\shortstack[c]{Manifold Exploration~\citep{de2025provable}\\[0.1em]Gen.\ model de-biasing}} &
  \multirow{2}{*}{$\entropy(p^\pi) \coloneqq -\EV_{x \sim p^\pi}[\log p^\pi(x)]$} &
  \multirow{2}{*}{\cross} & \multirow{2}{*}{\tick} & \multirow{2}{*}{\tick} \\[1.45em]
\rowcolor{lightblueII} \multirow{2}{*}{Risk-averse optimization} &
  $\mathrm{CVaR}^r_{\beta}(p^\pi) \coloneqq \EV_{x\sim p^\pi}[r(x) \mid r(x) \leq \mathrm{q}^r_{\beta}(p^\pi)]$ & \cross & \tick & \tick \\[0.6ex]
\rowcolor{lightblueII} &  $\mathbb{E}_{x \sim p^\pi}[r(x)] - \mathbb{V}\mathrm{ar}(p^\pi)$ & \cross & \cross & \tick \\
\rowcolor{lightblue}%
Novelty-seeking optimization &
  $\text{SQ}^r_{\beta}(p^\pi) \coloneqq \EV_{x\sim p^\pi}[r(x) \mid r(x) \geq \mathrm{q}^r_{\beta}(p^\pi)]$ & \cross & \cross & \tick \\
\rowcolor{lightblueII} \multirow{2}{*}{Optimal Experiment Design} &
  $\mathrm{s}\left(\EV_{x\sim p^\pi}[\Phi(x)\Phi(x)^\top - \lambda \mathbb{I}]\right)$ & 
  \multirow{2}{*}{\cross} & \multirow{2}{*}{\tick} & \multirow{2}{*}{\tick} \\[0.6ex]
\rowcolor{lightblueII} &  $\mathrm{s}(\cdot) \in \{\log \det(\cdot), -\mathrm{Tr}(\cdot)^{-1}, -\lambda_{max}(\cdot) \}$  &  &  &  \\
\rowcolor{lightblue} \rule[-1.5ex]{0pt}{4.0ex}
Diverse modes discovery &
  $- \EV_z[D_{KL}(p^{\pi, z} \| \EV_k p^{\pi, k})]$ & \cross & \cross & \tick \\
\rowcolor{lightblueII} Log-Barrier Constrained Generation &
  $\mathbb{E}_{x \sim p^\pi}[r(x)] - \beta \log \left(\langle p^\pi, c \rangle - C \right)$ &
  \cross & \tick & \tick \\
\midrule
\rowcolor{lightorange} \rule[-1.5ex]{0pt}{4.0ex} Kullback–Leibler divergence~\citep{domingo2024adjoint, uehara2024fine} &
  $D_{KL}(p^\pi \, \| \, p^{pre}) = \int p^\pi(x)\log \frac{p^\pi(x)}{p^{pre}(x)} \, dx$ & \tick & \tick & \tick \\
\rowcolor{lightorangeII} \rule[-1.5ex]{0pt}{4.0ex} Rényi divergences &
  $D_\beta(p^\pi \, \| \, p^{pre}) \coloneqq \frac{1}{\beta -1}\log \int (p^\pi(x))^\beta (p^{pre})^{1-\beta} \, dx$ &
  \cross & \cross & \tick \\
\rowcolor{lightorange} \rule[-1.5ex]{0pt}{4.0ex} Optimal Transport distances &
  $W_p(p^\pi \, \| \, p^{pre}) \coloneqq \inf_{\gamma \in \Gamma(p^\pi, p^{pre})} \EV_{(x,y) \sim \gamma}[d(x,y)^p]^{\frac{1}{p}}$ &
  \cross & \cross & \tick \\
\rowcolor{lightorangeII} \rule[-1.5ex]{0pt}{4.0ex} Maximum Mean Discrepancy &
  $\mathrm{MMD}_k(p^\pi \, \| \, p^{pre}) \coloneqq  \| \mu_{p^\pi} - \mu_{p^{pre}}\|$, $\mu_p \coloneqq \EV_{x\sim p}[k(x, \cdot)]$ &
  \cross & \tick & \tick \\
\bottomrule
\end{tabular}
}
\end{sc}
\endgroup
\caption{\looseness-1 Examples of practically relevant utilities $\F$ (blue) and divergences $\D$ (orange). Apx. \ref{sec:app_functionals_details} provides mathematical details and practical applications for each functional. Notice that besides $\entropy$, all non-linear functionals presented are novel in the context of fine-tuning of diffusion and flow models. }
\label{table:list_functionals}\vspace{-5mm}
\end{table*}
\vspace{-2.5mm}
\subsection{Beyond Linear Generative Optimization: an Expressivity Viewpoint}\vspace{-2.5mm}
\looseness -1

\looseness -1 Let $\G(p_1^\pi)\;=\;\F(p_1^\pi)\; - \;\alpha\,\D(p_1^\pi\!\,\|\,p_1^{\mathrm{pre}})$ be the functional in Eq. \eqref{eq:gen_opt_problem}. Then we denote by \emph{Convex} GO the case where $\G$ is concave in $p_1^\pi$, and by \emph{General} GO the case for arbitrary, possibly non-convex functionals \footnote{For clarity, we use the term \emph{convex} GO, rather than concave GO, to denote the problem class where concave functionals are optimized.}. In terms of expressivity \textbf{Linear GO} $\subset$ \textbf{Convex GO}  $\subset$ \textbf{General GO}, as depicted in Fig. \ref{fig:control_hierarchy} (left). In Table \ref{table:list_functionals} we classify into these tree tiers a sample of practically relevant utilities ($\F$, blue) and divergences ($\D$, orange).  In Apx. \ref{sec:app_functionals_details} we report complete definitions and applications. Except for entropy~\citep{de2025provable} and KL, all non-linear functionals in Table \ref{table:list_functionals} are to our knowledge explicitly used for the first time in the flow and diffusion model fine-tuning literature, while vastly employed in other areas. Moreover, the framework presented in this work for GO (Eq. \ref{eq:gen_opt_problem}) applies to any new choice of $\F$ or $\D$.

\looseness -1 Given the generality of generative optimization (Eq.\eqref{eq:gen_opt_problem}), a natural question arises: how can it be solved algorithmically? In the next section, we answer this by leveraging recent machinery from Convex~\citep{hazan2019maxent}
and General-Utilities RL~\citep{zhang2020variational}, to derive a fine‑tuning scheme that handles both convex and general GO, thus going beyond current entropy-regularized control methods, as illustrated in Fig. \ref{fig:control_hierarchy} (right).

\vspace{-2mm}
\section{Algorithm: \AlgNameLong} \vspace{-2mm}
\label{sec:algorithm}
In this section, we introduce \AlgNameDef, see Alg. \ref{alg:fdc_algorithm}, which provably solves the generative optimization problem in Eq. \eqref{eq:gen_opt_problem} via sequential fine-tuning of the pre-trained model $\pi^{pre}$. 
\looseness -1 To this end, we recall the notion of first variation of a functional over a space of probability measures~\cite{hsieh2019finding}. A functional $\G \in \mF(\X)$, where $\G: \mP(\X) \to \R$, has first variation at $\mu \in \mP(\X)$ if there exists a function $\delta \G(\mu) \in \mF(\X)$ such that for all $\mu' \in \mP(\X)$ it holds that:\vspace{-1mm}
\begin{equation*}
    \G(\mu + \epsilon \mu') = \G(\mu) + \epsilon \langle \mu', \delta \G(\mu) \rangle + o(\epsilon). 
\end{equation*}
\looseness -1 where the inner product has to be interpreted as an expectation. Intuitively, the first variation of $\G$ at $\mu$, namely $\delta \G(\mu)$, can be interpreted as an infinite-dimensional gradient in the space of probability measures. Given this notion, and a pair of generative models represented via policies $\pi$ and $\pi'$, we can now state the following \emph{entropy-regularized first variation maximization} fine-tuning problem.
 \vspace{-1.3mm}
\begin{tcolorbox}[colframe=white!, top=2pt,left=2pt,right=2pt,bottom=2pt]
\begin{center}
\textbf{Entropy-Regularized First Variation Maximization}
\begin{equation}
    \argmax_{\pi}\quad  \langle \delta \G \left( p^{\pi'}_1 \right), p_1^{\pi} \rangle - \eta D_{KL}(p_1^{\pi} \,\|\, p_1^{\pi'})\label{eq:opt_first_variation}
\end{equation}
\end{center}
\end{tcolorbox}\vspace{-1.3mm}
Crucially, we can introduce a function $g: \X \to \R$ defined for all $x \in \X$ such that:\vspace{-0.5mm}
\begin{equation}
    g(x) \coloneqq \delta \G \left(p^{\pi'}_1 \right) (x) \quad \text{ and } \quad \EV_{x \sim p^\pi}[g(x)] = \langle \delta \G \left( p^{\pi'}_1 \right), p_1^{\pi} \rangle  \label{eq:g_from_functional_linearized}
\end{equation}
\looseness -1 As a consequence, by rewriting Eq. \eqref{eq:opt_first_variation} expressing the first term via an expectation as shown in Eq. \eqref{eq:g_from_functional_linearized}, it corresponds to a common Linear GO problem (see Eq. \eqref{eq:linear_fine_tuning}), which can be optimized by utilizing established entropy-regularized control methods~\citep[\eg][]{uehara2024feedback, domingo2024adjoint, zhao2024scores}.

\begin{algorithm}[t] 
    \small
    \caption{\AlgNameDef} 
    \label{alg:fdc_algorithm}
        \begin{algorithmic}[1]
        \State{\textbf{input: } $\G: $ general utility functional, $K: $ number of iterations, $\pi^{pre}: $ pre-trained flow generative model, $\{\eta_k \}_{k=1}^{K}$ regularization coefficients}
        \State{\textbf{Init:} $\pi_0 \coloneqq \pi^{pre} $}
        \For{$k=1, 2, \hdots, K$}
        \State{Estimate: $\nabla_x g_{k} = \nabla_x \delta \G (p^{k-1}_1)$}
        \State{Compute $\pi_k$ via first-order linear fine-tuning:\vspace{-1mm}
            \begin{equation*}
                \pi_k \leftarrow \LinearFineTuningSolver(\nabla_x g_k, \eta_k, \pi_{k-1}) \vspace{-5pt}
            \end{equation*}
        }
        \EndFor
        \State{\textbf{output: } policy $\pi \coloneqq \pi_{K}$} 
        \end{algorithmic}
\end{algorithm}
\looseness -1  We can finally present \AlgNameDef, see Alg. \ref{alg:fdc_algorithm}, a mirror descent (MD) scheme~\citep{nemirovskij1983problem} that reduces optimization of non-linear functionals $\G$ to a specific sequence of Linear GO problems. \AlgNameShort takes three inputs: a pre-trained flow or diffusion model $\pi^{pre}$, the number of iterations $K$, and a sequence of regularization weights $\{\eta_k\}_{k=1}^K$. At each iteration, \AlgNameShort first estimates the gradient of the functional first variation at the previous policy $\pi_{k-1}$, \ie $\nabla_x \delta \G \left( p^{k-1}_1 \right)$ (line 4). Then, it updates the flow model $\pi_k$ by solving the fine-tuning problem in Eq. \eqref{eq:opt_first_variation} via an entropy-regularized control solver such as Adjoint Matching \citep{domingo2024adjoint}, using $\nabla_x g_k \coloneqq \nabla_x \delta \G \left( p^{k-1}_1 \right)$ as in Eq. \eqref{eq:g_from_functional_linearized} (line 5). Ultimately, it returns a final policy $\pi \coloneqq \pi_K$.
We report a detailed implementation of \AlgNameShort in Apx. \ref{sec:alg_implementation}.

\mypar{Gradient of first variation: computation and estimation}
\looseness -1 Surprisingly, estimating $\nabla_x g_k$ in Alg. \ref{alg:fdc_algorithm} (line 4) rarely requires density estimation. Among the functionals in Table \ref{table:list_functionals}, only the Rényi divergence does, for which one can leverage the recent Itô density estimator~\citep{skreta2024superposition}.  All other functionals admit straightforward plug-in or sample-based approximations detailed in Apx. \ref{sec:app_functionals_details}. As an illustrative example, in the following we showcase three examples from Table \ref{table:list_functionals}:
\begin{equation*}
    \nabla_x \delta \Q(p^\pi)(x) = 
    \begin{cases}
        -\nabla_x \log p^\pi(x) & \text{Entropy ($\entropy$)}\\
        \nabla_x r(x) \cdot \mathbf{1} \{r(x) \leq q_\beta^r(p^\pi)\} & \text{CVaR} \\
        \nabla_x \phi^*(x) \text{ where } \phi^* = \argmax_{\phi: \|\nabla_x \phi\|\leq 1} \langle \phi, p^\pi - p^{pre}\rangle & \text{Wasserstein-1 ($W_1$)}
    \end{cases}
\end{equation*}
\looseness -1 Here $\Q$ denotes either a utility $\F$ or a divergence $\D$, and $q_\beta^r(p^\pi)$ is the $\beta$-quantile of $Z = r(X)$ with $X \sim p^\pi$~\citep{rockafellar2002conditional}.
These gradients can be easily implemented. For entropy, the score term can be approximated via the score network in the case of diffusion models~\citep{de2025provable}, and obtained via a known linear transformation of the learned velocity field in the case of flows~\citep[Eq.(8)]{domingo2024adjoint}. For CVaR, any standard sample-based estimator of $q_\beta^r(p^\pi)$~\citep{rockafellar2002conditional} can be used. For Wasserstein-1, $\phi^*$ actually corresponds to the discriminator in Wasserstein-GAN, which can be learned with established methods~\citep{arjovsky2017wassersteingan}. In Apx. \ref{sec:app_functionals_details}, we report the gradient of the first variation for all functionals in Table \ref{table:list_functionals}, explain their practical  estimation,  and present a tutorial to derive the first variation of any new functionals not mentioned within Table \ref{table:list_functionals}.

\looseness -1 Given the approximate gradient estimates and the generality of the objective functions, it is still unclear whether the proposed algorithm provably converges to the optimal flow model $\pi^*$. In the next section, we answer this question by developing a theoretical analysis via recent results on mirror flows~\citep{hsieh2019finding}.

\vspace{-2mm}
\section{Guarantees for Generative Optimization via \AlgNameLong} \vspace{-2mm}
\label{sec:theory}
In this section, we recast \eqref{eq:gen_opt_problem} as \emph{constrained} optimization over stochastic processes, where the constraint is given by the Continuity Equation \eqref{eq:continuity_equation}. This formulation enables the application of \textbf{mirror descent for constrained optimization} and the notion of \emph{relative smoothness}~\cite{aubin2022mirror}. In our framework, convergence speed is governed by:
\begin{enumerate*}
    \item the structural complexity of the functional $\G$ (cf.~\cref{sec:algorithm}),
    \item the accuracy of the estimator $g$ from \eqref{eq:g_from_functional_linearized}, and
    \item the quality of the oracle \LinearFineTuningSolver in Alg. \ref{alg:fdc_algorithm}.
\end{enumerate*} To handle these cases, we will analyze two representative regimes: \vspace{-6mm}
\begin{itemize}
    \item \textbf{Idealized.} $\G$ is \emph{concave}, and both $g$ and \LinearFineTuningSolver are exact. In this setting, classical results yield sharp step-size prescriptions and fast convergence rates.
    \item \textbf{General.} $\G$ is  \emph{non-concave}, with $g$ and the oracle subject to noise and bias. While fast convergence is generally out of reach~\cite{mertikopoulos2024unified, karimi2024sinkhorn}, convergence to a stationary point remains attainable under mild assumptions.\vspace{-3.5mm}
\end{itemize}

\para{Theoretical analysis: Idealized setting}

We now present a framework leading to convergence guarantees for \AlgNameShort (\ie Alg. \ref{alg:fdc_algorithm}) for \emph{concave} functionals $\G \in \mF(\X)$. We start by recalling the notion of Bregman divergence induced by a functional $\Q \in \mF(\X)$ between densities $\mu, \nu \in \mP(\X)$, namely:\vspace{-1mm}
\begin{equation*}
    D_\Q(\mu \, \| \, \nu) \coloneqq \Q(\mu) - \Q(\nu) - \langle \delta \Q(\nu), \mu - \nu \rangle 
\end{equation*}
\looseness -1 Next, we introduce two structural properties for our analysis.
\begin{restatable}[Relative smoothness and relative strong concavity \citep{lu2018relatively}]{definition}{relativeProperties}
\label{definition:relative_properties}
Let $\G: \mP(\X) \to \R$ a concave functional. We say that $\G$ is L-smooth relative to $\Q \in \mF(\X)$ over $\mP(\X)$ if $\exists$ $L$ scalar s.t. for all $\mu, \nu \in \mP(\X)$:\vspace{-2mm}
\begin{equation}
    \G(\nu) \geq \G(\mu) + \langle \delta \G(\mu), \nu - \mu \rangle - LD_\Q(\nu \, \| \, \mu)
\end{equation}
and we say that $\G$ is l-strongly concave relative to $\Q \in \mF(\X)$ over $\mP(\X)$ if $\exists$ $l \geq 0$ scalar s.t. for all $\mu, \nu \in \mP(\X)$:\vspace{-1mm}
\begin{equation}
    \G(\nu) \leq \G(\mu) + \langle \delta \G(\mu), \nu - \mu \rangle - lD_\Q(\nu \, \| \, \mu)
\end{equation}
\end{restatable}

\looseness -1 In the following, we interpret line $(6)$ of \AlgNameShort as a step of mirror ascent~\citep{nemirovskij1983problem}, and the KL divergence term as the Bregman divergence induced by an entropic mirror map $\Q = \entropy$, \ie $D_{KL}(\mu, \nu) = D_\entropy(\mu \, \| \, \nu)$.
We can finally state the following set of  assumptions as well as the convergence guarantee for an arbitrary functional $\G(\cdot) = \F(\cdot) - \alpha \D(\cdot \, \| \, p^{pre}) \in \mF(\X)$.

\begin{restatable}[Exact estimation and optimization]{assumption}{allAssumptions}
\label{assumption:all_assumptions}
We consider the following assumptions:\vspace{-1mm}
\begin{enumerate}
    \item Exact estimation: $\nabla_x \delta \G(p^k_1)$ is estimated exactly $\forall k \in [K]$. \vspace{-1mm}
    \item The optimization problem in Eq. \eqref{eq:opt_first_variation} is solved exactly.\vspace{-1mm}
\end{enumerate}
\end{restatable}
\vspace{-1.3mm}
\begin{tcolorbox}[colframe=white!, top=2pt,left=2pt,right=2pt,bottom=2pt]
\begin{restatable}[Convergence guarantee of \AlgNameLong with concave functionals]{theorem}{convexCaseConvergence}
\label{theorem:convex_case_convergence}
Given Assumptions \ref{assumption:all_assumptions}, fine-tuning a pre-trained model $\pi^{pre}$ via \AlgNameShort (Algorithm \ref{alg:fdc_algorithm}) with $\eta_k = L$ $\forall k \in [K]$, leads to a policy $\pi$ inducing a marginal distribution $p^{\pi}_1$ such that:
\begin{equation}
    \G(p_1^*) - \G(p^{\pi}_1) \leq \frac{L-l}{K}D_{KL}(p_1^* \, \| \, p_1^{pre})
\end{equation}
where $p_1^* \coloneqq p_1^{\pi^*}$ is the marginal distribution induced by the optimal policy $\pi^* \in \argmax_{\pi} \G(p^{\pi}_1) \coloneqq \F(p^{\pi}_1) - \alpha \D(p^{\pi}_1 \, \| \, p^{pre}_1)$.
\end{restatable}
\end{tcolorbox}
\vspace{-1.3mm}
\looseness -1 \cref{theorem:convex_case_convergence} provides a fast convergence rate under a specific step-size choice ($\eta_k = L$). However, it critically depends on \cref{assumption:all_assumptions}, which typically does not hold in practice. To address this limitation, we now consider a more general scenario where this key assumption is relaxed.
\vspace{-2mm}
\para{Theoretical analysis: General setting}
Recall that $p_1^k \coloneqq p^{\pi_k}_1$ represents the (stochastic) density produced by the \LinearFineTuningSolver oracle at the \(k\)-th step of \AlgNameShort, and consider the following \emph{mirror ascent} iterates, where $1 / \lambda_k = \eta_k$ in Algorithm \ref{alg:fdc_algorithm}:\vspace{-2mm}
\begin{equation}
\label{eq:MD}
\tag{MD$_k$}
    \pMD^{k} \coloneqq \argmax_{p \in \mathbb{P}(\Omega_{pre})} \quad \langle \delta \G \left(  p_T^{\pi_{k-1}} \right), p \rangle - \frac{1}{\step_k} D_{KL}(p \, \| \,  p_T^{\pi_{k-1}}) 
\end{equation}  
\looseness -1 In realistic settings, where only noisy \emph{and} biased approximations of \eqref{eq:MD} are available, it is essential to quantify the deviations from the idealized iterates in \eqref{eq:MD}. 
To this end, denote by $\hist_k$ the filtration up to step $k$, and consider the decomposition of the oracle into its \emph{noise} and \emph{bias} parts: 
\begin{align}
    \bias_k \coloneqq \EV \left[ \delta \G(p^{\pi_k}_T) -   \delta \G(\pMD^k)  \,  \vert\, \hist_k \right], \qquad 
    \noise_k \coloneqq\delta \G(p^{\pi_k}_T) -   \delta \G(\pMD^k)  -  \bias_k \vspace{-7mm}
\end{align}

Conditioned on \(\hist_k\), \(\noise_k\) has zero mean, while \(\bias_k\) captures the \emph{systematic} error. We then impose:
\begin{assumption}[Noise and Bias]
\label{asm:approximate}
The following events happen almost surely:
\begin{align}
    \norm{\bias_k}_\infty \rightarrow 0, \qquad
    \sum_k \EV \left[ \step_k^2 \left( \norm{\bias_k}_\infty^2 + \norm{\noise_k}_\infty^2 \right) \right] <\infty,\qquad
    \label{eq:bias-step}
    \sum_k{\step_k\norm{\bias_k}_\infty} < \infty
\end{align}
\end{assumption}\vspace{-3mm}
\looseness -1 The first condition is a \emph{necessary} requirement for convergence since when violated, it is easy to construct scenarios where no practical algorithm can solve the generative optimization problem. The second and third inequalities manage the trade-off between \emph{accuracy} of the approximate oracle \LinearFineTuningSolver{} and \emph{aggressiveness} of the step sizes, $\step_k$. Intuitively, lower noise and bias in the oracle enable the use of larger step sizes. To this end, \cref{asm:approximate} provides a concrete criterion that guarantees the success of finding the optimal policy with probability one.
\vspace{-1.3mm}
\begin{tcolorbox}[colframe=white!, top=2pt,left=2pt,right=2pt,bottom=2pt]
\begin{restatable}[Convergence guarantee of \AlgNameLong for general functionals]{theorem}{generalCaseConvergence}
\label{theorem:general_case_convergence}
\looseness -1 Given the Robbins-Monro step-size rule: $\sum_k \step_k = \infty, \sum_k \step_k^2 < \infty$, under \cref{asm:approximate} and technical assumptions (see \cref{sec:app-theory2}), the sequence of marginal densities $p^k_1$ induced by the iterates $\pi_k$ of \cref{alg:fdc_algorithm} 
converges weakly to a stationary point  $\tilde{p_1}$ of $\G$ almost surely, formally:  $p_1^k \rightharpoonup \tilde{p_1} \quad \text{a.s.}$.
\end{restatable}
\end{tcolorbox}
\vspace{-1.3mm}

\vspace{-2.5mm}
\section{Experimental Evaluation} \vspace{-3mm}
\label{sec:experiments}
\looseness -1 We analyze the ability of \AlgNameDef $ $ to induce policies optimizing complex non-linear objectives, and compare its performance with Adjoint Matching (\AlgNameShortAM)~\citep{domingo2024adjoint}, a classic fine-tuning method. We present two types of experiments:  (i) Illustrative settings to provide insights via visual interpretability, and (ii) High-dimensional real-world applications, namely (a) novelty-seeking molecular design for single-point energy minimization~\citep{ friede2024dxtb}, and (b)  manifold exploration for text-to-image \emph{creative bridge design} generation. Additional details are provided in Apx. \ref{sec:experimental_detail}.

\begingroup
  \captionsetup[subfigure]{aboveskip=1.7pt, belowskip=0pt}
\newlength{\imgw}
\setlength{\imgw}{0.25\textwidth}
\begin{figure*}[t]
    \centering
    \begin{subfigure}{\imgw}
      \centering
      \includegraphics[width=\textwidth]{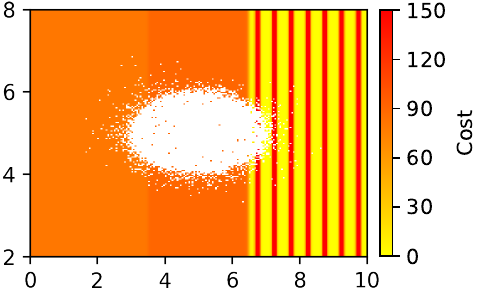}
      \caption{Pre-trained samples}
      \label{fig:toy_top_a}
    \end{subfigure}%
    \begin{subfigure}{\imgw}
      \centering
      \includegraphics[width=\textwidth]{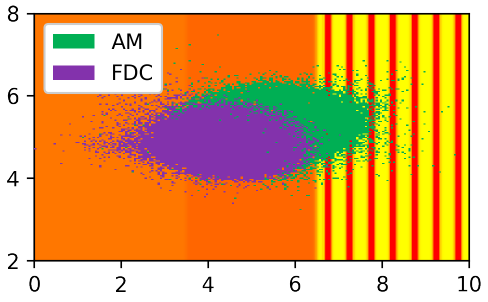}
      \caption{\AlgNameShortAM vs \AlgNameShort samples}
      \label{fig:toy_top_b}
    \end{subfigure}%
    \begin{subfigure}{\imgw}
      \centering
      \includegraphics[width=\textwidth]{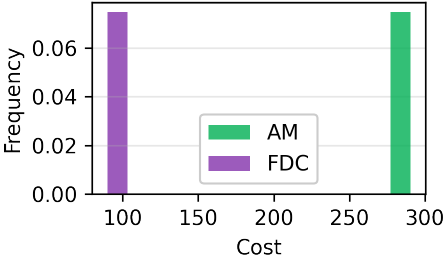}
      \caption{$\beta$-worst-case costs}
      \label{fig:toy_top_c}
    \end{subfigure}%
    \begin{subfigure}{\imgw}
      \centering
      \includegraphics[width=\textwidth]{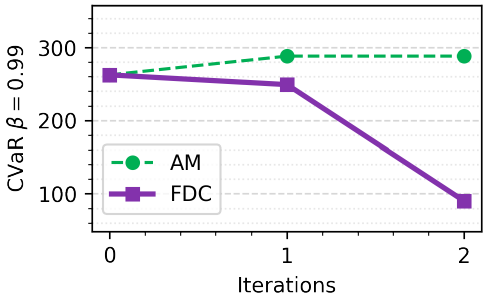}
      \caption{CVaR evaluation}
      \label{fig:toy_top_d}
    \end{subfigure}%
    \\[0.4em]
    \begin{subfigure}{\imgw}
      \centering
      \includegraphics[width=\textwidth]{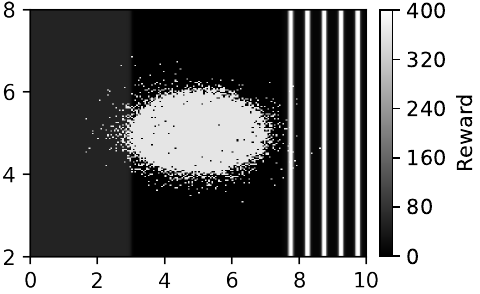}
      \caption{Pre-trained samples}
      \label{fig:toy_mid_a}
    \end{subfigure}%
    \begin{subfigure}{\imgw}
      \centering
      \includegraphics[width=\textwidth]{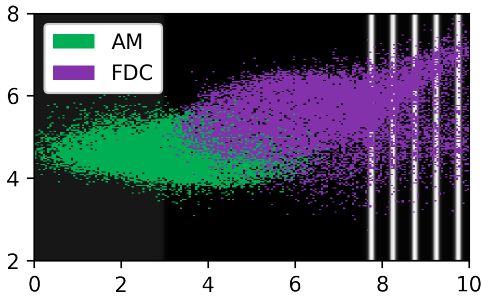}
      \caption{\AlgNameShortAM vs \AlgNameShort samples}
      \label{fig:toy_mid_b}
    \end{subfigure}%
    \begin{subfigure}{\imgw}
      \centering
      \includegraphics[width=\textwidth]{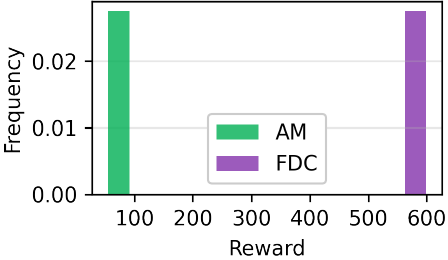}
      \caption{$\beta$-best-case rewards}
      \label{fig:toy_mid_c}
    \end{subfigure}%
    \begin{subfigure}{\imgw}
      \centering
      \includegraphics[width=\textwidth]{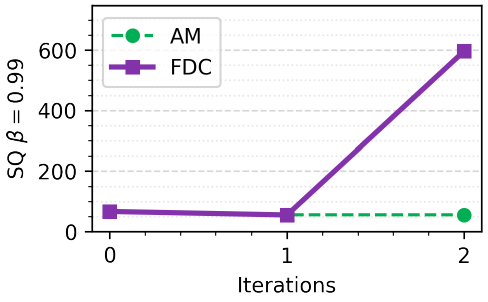}
      \caption{SQ evaluation}
      \label{fig:toy_mid_d}
    \end{subfigure}%
    \\[0.4em]
    \begin{subfigure}{\imgw}
      \centering
      \includegraphics[width=\textwidth]{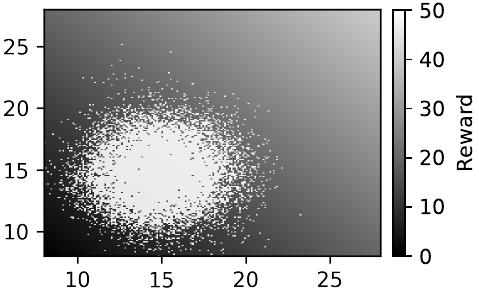}
      \caption{Pre-trained samples}
      \label{fig:toy_bottom_a}
    \end{subfigure}%
    \begin{subfigure}{\imgw}
      \centering
      \includegraphics[width=\textwidth]{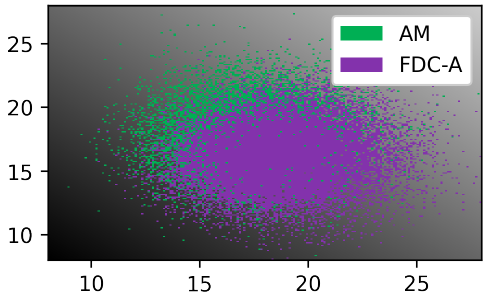}
      \caption{\AlgNameShortAM vs \AlgNameShort-A samples}
      \label{fig:toy_bottom_b}
    \end{subfigure}%
    \begin{subfigure}{\imgw}
      \centering
      \includegraphics[width=\textwidth]{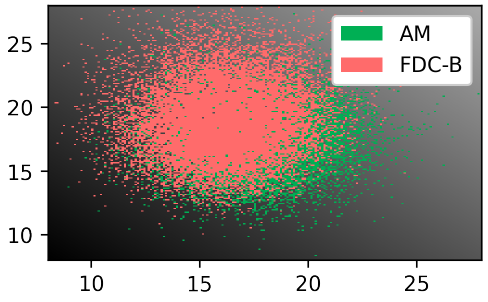}
      \caption{\AlgNameShortAM vs \AlgNameShort-B samples}
      \label{fig:toy_bottom_c}
    \end{subfigure}%
    \begin{subfigure}{\imgw}
      \centering
      \setlength{\tabcolsep}{3pt}    
      \renewcommand{\arraystretch}{0.85} 
      \begin{adjustbox}{valign=c}
      \raisebox{6.5ex}[0pt][0pt]{
      \resizebox{0.9\textwidth}{!}{%
        \begin{tabular}{@{}lccc@{}}
          \toprule
          & $\EV[r(x)]$ & $W_1^A$ & $\Delta \%$ \\
          \midrule
          Pre & $29.5$ & $0$ & $-$ \\
          \AlgNameShortAM & $35.0$ & $4.67$ & $100$ \\
          \AlgNameShort-A & $35.4$ & $1.95$ & $280$ \\
          \bottomrule
        \end{tabular}%
      }
      }
      \end{adjustbox}
      \caption{$W_1^A$ evaluation}
      \label{fig:toy_bottom_d}
    \end{subfigure}%
    \vspace{0pt}
    \caption{\looseness-1 Illustrative settings with visually interpretable results. (top) Risk-averse reward maximization for valid or safe generation, (mid) Novelty-seeking reward maximization for discovery, (bottom) Expected rewards maximization under optimal transport distance regularization. Crucially, \AlgNameShort can optimize well these complex objectives, while \AlgNameShortAM~\citep{domingo2024adjoint}, a classic fine-tuning scheme, fails at this.} \vspace{-6.5mm}
    \label{fig:experiments_fig_1}
\end{figure*}
\endgroup
\vspace{-0.5mm}
\mypar{Risk-averse reward maximization for better worst-case validity or safety}
\looseness -1 We fine-tune a pre-trained policy $\pi^{pre}$ (see Fig. \ref{fig:toy_top_a}) by optimizing the CVaR$_\beta$ utility \ie expected outcome in the $\beta$-worst-case (see Tab. \ref{table:list_functionals}) with KL regularization, and costs interpreted as negative rewards. The cost has three regions: a high-cost plateau (dark orange), where the initial density lies; a moderate-cost left area (light orange); and a predominantly low-cost right zone (yellow) punctuated by narrow, but catastrophic red-stripes. As shown in Fig. \ref{fig:toy_top_b}, \AlgNameShortAM moves the model density into the yellow region, lowering average cost but exposing it to rare extreme costs. In contrast, \AlgNameShort, run with $K=2$ iterations and $\beta=0.01$, successfully steers density into the safer, moderate-cost area, cutting the 1\%-worst-case cost from $288.2$ achieved by \AlgNameShortAM to $90.0$, well below the initial $262.5$, as shown in Fig. \ref{fig:toy_top_c} and \ref{fig:toy_top_d}. 
 
\vspace{-0.5mm}
\mypar{Novelty-seeking reward maximization for discovery}
\looseness -1 We fine‐tune a pre‐trained policy $\pi^{pre}$ to maximize the SQ$_\beta$ utility, \ie expected outcome in the $\beta$-best-case (see Tab. \ref{table:list_functionals}). The reward shown in Fig. \ref{fig:toy_mid_a} has a moderately high‐reward left area (light gray), a medium‐reward central plateau (darker gray) where the initial density lies, and a low-reward right region (black) with sparse, extreme‐reward spikes depicted by thin white lines. As shown in Fig. \ref{fig:toy_mid_b}, \AlgNameShortAM drifts the density into the safer left basin — improving the average reward but only reaching a best‐1\% expected reward of $55.5$, as shown in Fig. \ref{fig:toy_mid_c} and Fig. \ref{fig:toy_mid_d}. In contrast, \AlgNameShort, run for $K=2$ iterations and $\beta=0.99$, pushes the density rightwards, elevating the top‐1\% reward to $596.1$ (see Fig. \ref{fig:toy_mid_d}) — far above both \AlgNameShortAM and the initial $66.6$.

\vspace{-0.5mm}
\mypar{Reward maximization regularized via optimal transport distance}
\looseness -1 We fine‐tune the pre‐trained model with density in Fig. \ref{fig:toy_bottom_a} to maximize a reward function that increases moving top right. We consider two $W_{1}$ distances induced by two ground metrics: $d_{A}$, which makes vertical moves more costly than horizontal ones, and $d_{B}$, which does the opposite. Under $d_{A}$, both AM and the OT‐regularized model reach an expected reward of $35.0$, but FDC-A incurs only $W_{1}^{A}=1.95$ versus $4.67$ for AM, and achieves a mean shift that is $280 \%$ larger in the horizontal than in the vertical direction (Fig. \ref{fig:toy_bottom_b} and Tab. \ref{fig:toy_bottom_d}). By contrast, FDC-B under $d_{B}$ preferentially shifts the density upward (Fig. \ref{fig:toy_bottom_c}).

\vspace{-0.5mm}
\mypar{Conservative manifold exploration}
\looseness -1 We tackle manifold exploration~\citep{de2025provable} by fine-tuning a pre-trained model $\pi^{pre}$ to maximize the entropy utility ($\entropy$ in Tab. \ref{table:list_functionals}) under a KL regularization of strength $\alpha$, a capability not possible with prior methods~\citep{de2025provable}. As in previous work, we consider the common setting where the pre-trained model density $p^{pre}_1$  concentrates most of its mass in a specific region as shown in Fig. \ref{fig:real_top_a}, where $N=10000$ samples are shown. By fine-tuning $\pi^{pre}$ via \AlgNameShort, the density of the fine-tuned model shifts into low-coverage areas (see Fig. \ref{fig:real_top_b} and \ref{fig:real_top_c}). In particular, Fig. \ref{fig:real_top_d} demonstrates that reducing $\alpha$ from $0.5$ to $0.0$ yields progressively higher Monte Carlo entropy estimates ($7.00$ at $\alpha=0.5$, $7.14$ at $\alpha=0$), thus enabling control of the trade-off between preserving the original distribution and exploring novel regions, a capability not supported by prior methods~\citep{de2025provable}.

\begingroup
  \captionsetup[subfigure]{aboveskip=1.7pt, belowskip=0pt}
\setlength{\imgw}{0.25\textwidth}
\newlength{\sdw}
\setlength{\sdw}{0.45\imgw}
\begin{figure*}[t]
    \centering
    \begin{subfigure}{\imgw}
      \centering
      \includegraphics[width=\textwidth]{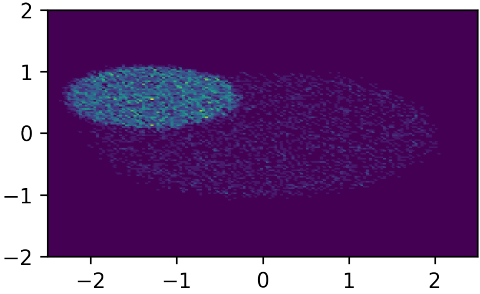}
      \caption{Pre-trained samples}
      \label{fig:real_top_a}
    \end{subfigure}%
    \begin{subfigure}{\imgw}
      \centering
      \includegraphics[width=\textwidth]{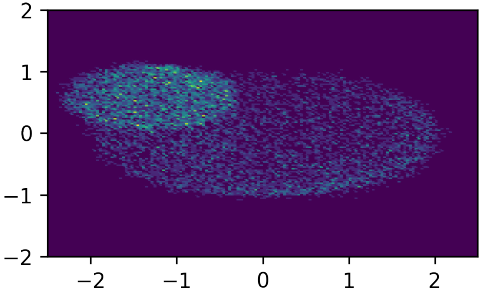}
      \caption{\AlgNameShort $\alpha=0.5$ samples}
      \label{fig:real_top_b}
    \end{subfigure}%
    \begin{subfigure}{\imgw}
      \centering
      \includegraphics[width=\textwidth]{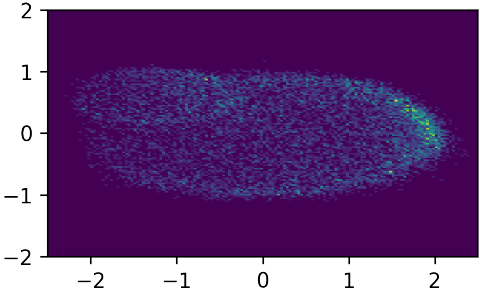}
      \caption{\AlgNameShort $\alpha=0.0$ samples}
      \label{fig:real_top_c}
    \end{subfigure}%
    \begin{subfigure}{\imgw}
      \centering
      \setlength{\tabcolsep}{3pt}    
      \renewcommand{\arraystretch}{0.85} 
      \begin{adjustbox}{valign=c}
      \raisebox{6.5ex}[0pt][0pt]{
      \resizebox{0.8\textwidth}{!}{%
        \begin{tabular}{@{}ll@{}}
            \toprule
               & $\entropy(p^\pi)$               \\
            \midrule
            Pre-trained    $\;$  & \; 6.78\\
            \AlgNameShort $\alpha = 0.5$ &\; 7.00\\
            \AlgNameShort $\alpha = 0.0$ &\; 7.14\\
            \bottomrule
            \end{tabular}
      }
      }
      \end{adjustbox}
      \caption{Entropy evaluation}
      \label{fig:real_top_d}
    \end{subfigure}%
    \\[0.5em] 
    \begin{subfigure}{3\sdw}
      \centering
      \includegraphics[width=.32\textwidth]{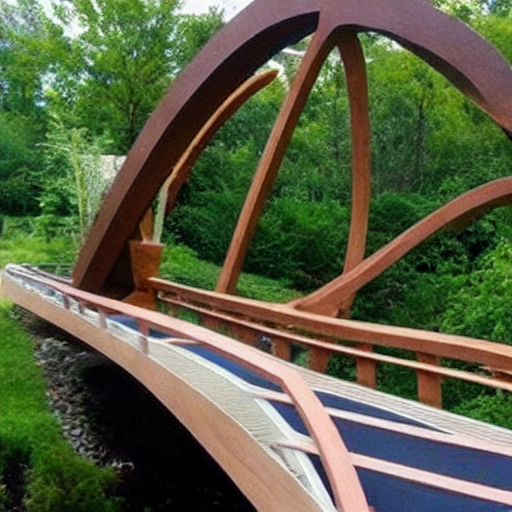}%
    \includegraphics[width=.32\textwidth]{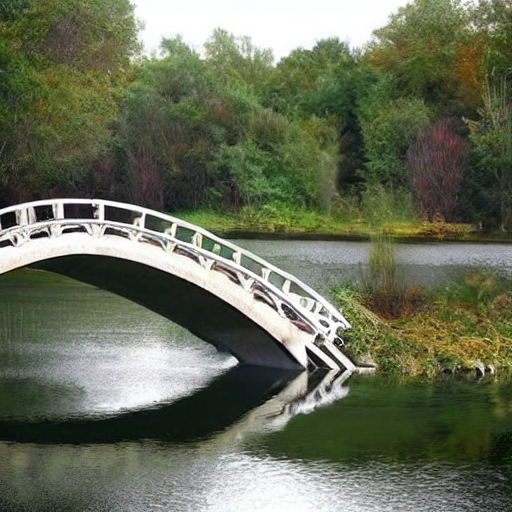}%
  \includegraphics[width=.32\textwidth]{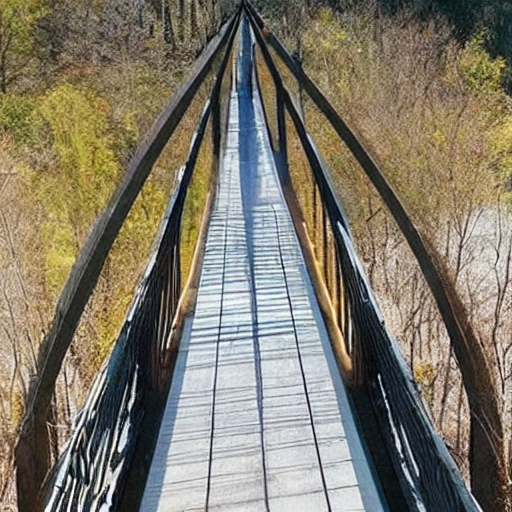}
      \caption{Pre-trained samples}
      \label{fig:real_mid_a}
    \end{subfigure}%
    \begin{subfigure}{0.5\sdw}
        \centering
        $\;\Rightarrow$
        \vspace{2.6em}
    \end{subfigure}
    \begin{subfigure}{3\sdw}
      \centering
      \includegraphics[width=.32\textwidth]{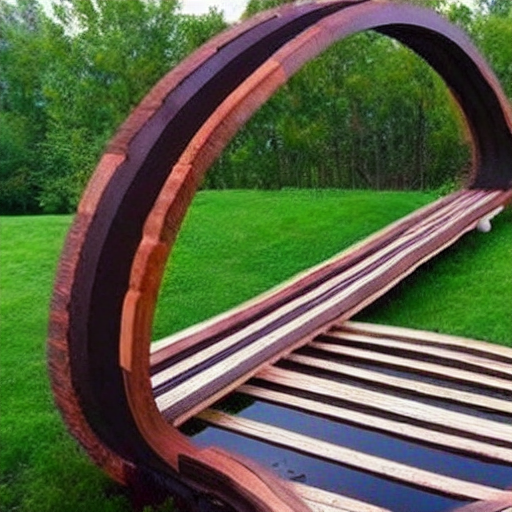}%
      \includegraphics[width=.32\textwidth]{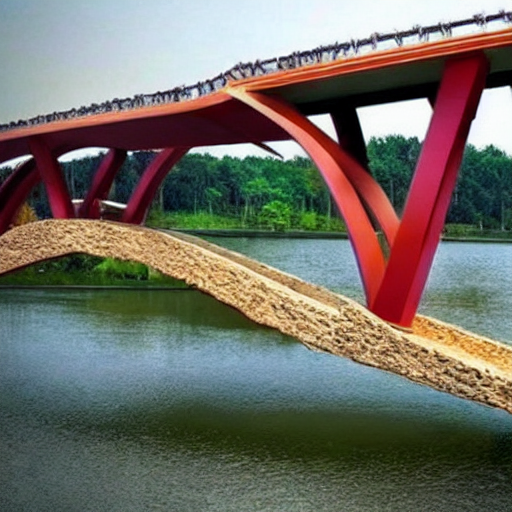}%
      \includegraphics[width=.32\textwidth]{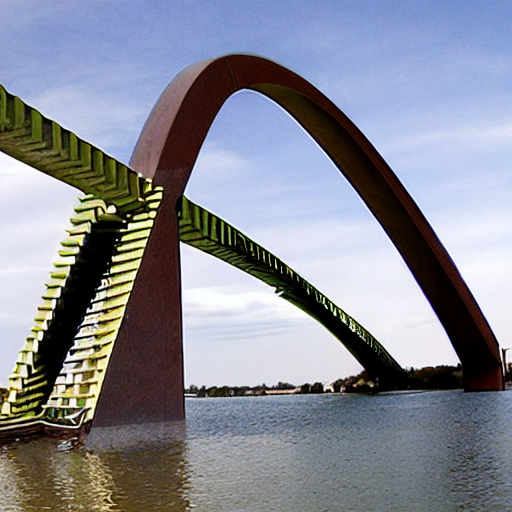}
      \caption{\AlgNameShort samples}
      \label{fig:real_mid_b}
    \end{subfigure}%
    \begin{subfigure}{\imgw}
      \centering
      \setlength{\tabcolsep}{3pt}    
      \renewcommand{\arraystretch}{0.85} 
      \begin{adjustbox}{valign=c}
      \raisebox{3.6ex}[0pt][0pt]{
      \resizebox{0.9\textwidth}{!}{%
        \begin{tabular}{@{}lll@{}}
            \toprule
               & Vendi  & CLIP         \\
            \midrule
            Pre-trained    $\;$  & \; $2.36$ & $0.19$ \\
            \AlgNameShort $\alpha = 0.001$ &\; $\mathbf{2.47}$ & $\mathbf{0.22}$\\
            \bottomrule
            \end{tabular}
      }
      }
      \end{adjustbox}
      \caption{Images evaluation}
      \label{tab-sd-vendi}
    \end{subfigure}%
    \\[0.5em]
    \begin{subfigure}{\imgw}
      \centering
      \includegraphics[width=\textwidth]{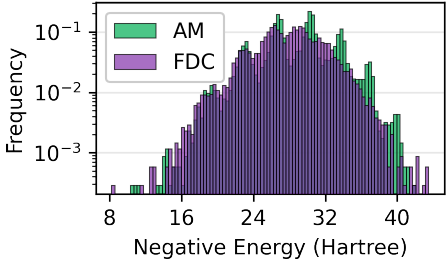}
      \caption{Energy distributions}
      \label{fig:real_bottom_a}
    \end{subfigure}%
    \begin{subfigure}{\imgw}
      \centering
      \includegraphics[width=\textwidth]{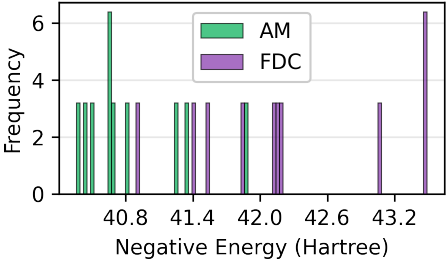}
      \caption{Top $0.2\%$ Energies}
      \label{fig:real_bottom_b}
    \end{subfigure}%
    \begin{subfigure}{\imgw}
      \centering
      \setlength{\tabcolsep}{3pt}    
      \renewcommand{\arraystretch}{0.85} 
      \begin{adjustbox}{valign=c}
      \raisebox{7.0ex}[0pt][0pt]{
      \resizebox{0.85\textwidth}{!}{%
        \begin{tabular}{@{}lccc@{}}
          \toprule
          & $\EV[r(x)]$ & $SQ_{\beta}$ \\
          \midrule
          Pre & $15.4$ & $24.2$ \\
          \AlgNameShortAM & $29.1$ & $39.7$ \\
          \AlgNameShort-A & $27.5$ & $41.8$ \\
          \bottomrule
        \end{tabular}%
      }
      }
      \end{adjustbox}
      \caption{SQ Negative Energy}
      \label{fig:real_bottom_c}
    \end{subfigure}%
    \begin{subfigure}{\imgw}
      \centering
      \raisebox{0.35ex}[0pt][0pt]{
      \includegraphics[width=0.8\textwidth]{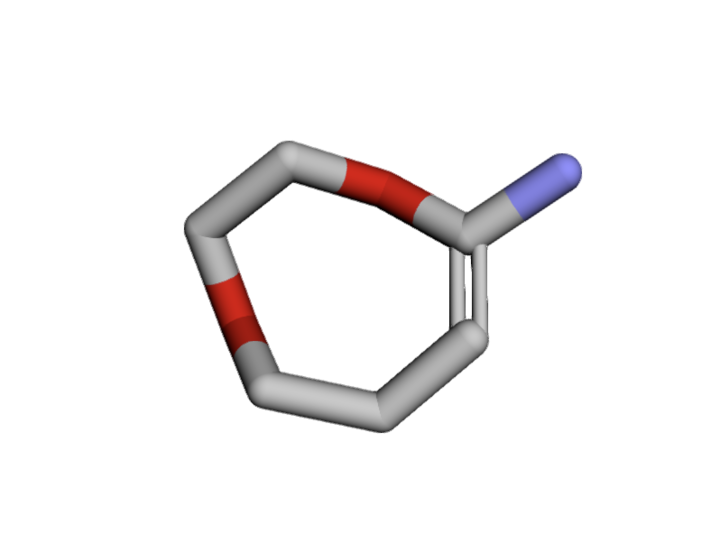}
      }
      \caption{\AlgNameShort -generated design}
      \label{fig:real_bottom_d}
    \end{subfigure}%
    \vspace{-0mm}
    \caption{\looseness-1 (top) Illustrative manifold exploration experiment via KL-regularized entropy maximization, (mid) High-dimensional manifold exploration via text-to-image model fine-tuning for prompt "A creative bridge design". Left: images from pre-trained model, Right: images from model fine-tuned via \AlgNameShort, with higher diversity as indicated by a higher Vendi score. (bottom) Novelty-seeking molecular design for Energy (kcal/mol) maximization by fine-tuning FlowMol~\citep{dunn2024mixed}. \AlgNameShort shows enhanced control capabilities for optimizing such complex objectives than \AlgNameShortAM, a classic fine-tuning scheme.}
    \label{fig:experiments_fig_2}
    \vspace{-6mm}
\end{figure*}
\endgroup
\vspace{-0.5mm}
\mypar{Molecular design for single-point energy minimization}
\looseness -1 We fine-tune FlowMol~\citep{dunn2024mixed}, pre-trained on QM9~\citep{ramakrishnan2014quantum}, to discover molecules minimizing the single-point total energy computed via extended tight-binding at the GFN1‐xTB level of theory~\citep{friede2024dxtb}. Concretely, we maximize the negative energy. We do not aim to maximize the average sample reward, but rather that of the top $0.2\%$ samples. We employ \AlgNameShort with novelty-seeking SQ utility (see Tab. \ref{table:list_functionals}) with $\beta = 0.998$, and make $2$ gradient steps per $K=10$ iterations. We compare it with \AlgNameShortAM run for $240$ steps. Fig. \ref{fig:real_bottom_c} shows that while \AlgNameShortAM generates better samples in average (namely $29.1$ over $27.5$ of \AlgNameShort), the average quality of the top $0.2\%$ molecules, indicated by SQ$_\beta$ is higher for \AlgNameShort than for \AlgNameShortAM (namely $41.8$ over $39.7$ of \AlgNameShortAM). This confirms (see Fig. \ref{fig:real_bottom_b} and \ref{fig:real_bottom_a}) that \AlgNameShort can sacrifice the average reward to generate a few truly high-reward designs.

\vspace{-0.5mm}
\mypar{Text-to-image bridge designs conservative exploration}
\looseness -1 We perform manifold exploration by fine-tuning Stable Diffusion (SD) $1.4$~\citep{rombach2021highresolution} with prompt "A creative bridge design.". To this end, we maximize the KL-regularized entropy (see Tab. \ref{table:list_functionals}) with $\alpha = 0.001$ via \AlgNameShort for $K=2$ steps.
As a diversity metric, we utilize the Vendi score~\citep{friedman2022vendi} with cosine similarity kernel on the extracted CLIP~\citep{hessel2021clipscore} features from a sample of $100$ images and compared it to the baseline pre-trained model in Fig. \ref{tab-sd-vendi}. Beyond increasing the Vendi score, \AlgNameShort also increases the CLIP score of the initial model.

\vspace{-2mm}
\section{Related Works} \vspace{-2mm}
\label{sec:related_works}
\mypar{Flow and diffusion models fine-tuning via optimal control}
\looseness -1 Recent works have framed fine-tuning of diffusion and flow models to maximize expected reward under KL regularization as an entropy-regularized optimal control problem~\citep[\eg][]{uehara2024fine, tang2024fine, uehara2024feedback,  domingo2024adjoint}. Crucially, as shown in Sec. \ref{sec:problem_setting}, the problem tackled by these studies is the specific sub-case of generative optimization (Eq. \eqref{eq:gen_opt_problem}), where the utility $\F$ is linear, and $\D = D_{KL}$. In this work, we propose a principled method with guarantees for the far more general class of non-linear utilities and divergences beyond KL, including the ones listed in Tab. \ref{table:list_functionals}. The framework introduced has strictly higher expressive power and control capabilities for fine-tuning generative model (see Sec. \ref{sec:problem_setting}). This renders possible to tackle relevant tasks \eg scientific discovery, beyond the capabilities of the aforementioned fine-tuning schemes.

\vspace{-1mm}
\mypar{Convex and General Utilities Reinforcement Learning}
\looseness -1 Convex and General (Utilities) RL~\citep{hazan2019maxent, zahavy2021reward, zhang2020variational} generalizes RL to the case where one wishes to maximize a concave~\citep{hazan2019maxent, zahavy2021reward}, or general~\citep{zhang2020variational, barakat2023reinforcement} functional of the state distribution induced by a policy over a dynamical system's state space. The introduced generative optimization problem (in Eq. \eqref{eq:gen_opt_problem}) is related, with $p^\pi_1$ representing the state distribution induced by policy $\pi$ over a subset of the state space. Recent works tackled the finite samples budget setting~\citep[\eg][]{mutti2022importance, mutti2022challenging, mutti2023convex, prajapat2023submodular, de2024global}. Ultimately, to our knowledge, this is the first work leveraging an algorithmic scheme resembling General RL for the practically relevant task of generative optimization of general non-linear functionals via fine-tuning of diffusion and flow models.

\vspace{-1mm}
\mypar{Optimization over probability measures via mirror flows}
\looseness -1 Recently, there has been a growing interest in building theoretical guarantees for optimization problems over spaces of probability measures in a variety of applications. These include GANs~\cite{hsieh2019finding}, optimal transport~\cite{aubin2022mirror, leger2021gradient, karimi2024sinkhorn}, kernelized methods~\cite{dvurechensky2024analysis}, and manifold exploration~\citep{de2025provable}. We present the first use of this framework to establish guarantees for the generative optimization problem in Eq. \eqref{eq:gen_opt_problem}. This novel link to probability-space optimization sheds new light on large-scale flow and diffusion models fine-tuning.

\vspace{-2.5mm}
\section{Conclusion} \vspace{-3mm}
\label{sec:conclusions}
\looseness -1 This work tackles the fundamental challenge of fine-tuning pre-trained flow and diffusion generative models on arbitrary task-specific utilities and divergences while retaining prior knowledge. We introduce a unified generative optimization framework that strictly generalizes existing formulations and propose a rich class of new practically relevant objectives. We then propose \AlgNameLong, a mirror-descent algorithm that reduces complex generative optimization to a sequence of standard fine-tuning steps, each solvable by scalable off-the-shelf methods. Leveraging convex analysis and recent advances in mirror flows theory, we prove convergence under general conditions. Empirical results on synthetic benchmarks, molecular design, and image generation, demonstrate that our approach can steer pre-trained models to optimize objectives beyond the reach of current fine-tuning techniques. As for limitations, while our framework is general, future work will need to assess to what extent the flexibility in selecting utilities and divergences yields concrete gains in specific applications.

\section*{Acknowledgements}
This publication was made possible by the ETH AI Center doctoral fellowship to Riccardo De Santi, and postdoctoral fellowship to Marin Vlastelica. The project has received funding from the Swiss
National Science Foundation under NCCR Catalysis grant number 180544 and NCCR Automation grant agreement 51NF40 180545.

\bibliography{biblio}

@article{milgrom2002envelope,
  title={Envelope theorems for arbitrary choice sets},
  author={Milgrom, Paul and Segal, Ilya},
  journal={Econometrica},
  volume={70},
  number={2},
  pages={583--601},
  year={2002},
  publisher={Wiley Online Library}
}

@article{rockafellar2000optimization,
  title={Optimization of conditional value-at-risk},
  author={Rockafellar, R Tyrrell and Uryasev, Stanislav and others},
  journal={Journal of risk},
  volume={2},
  pages={21--42},
  year={2000},
  publisher={Citeseer}
}

@book{hiriart2004fundamentals,
  title={Fundamentals of convex analysis},
  author={Hiriart-Urruty, Jean-Baptiste and Lemar{\'e}chal, Claude},
  year={2004},
  publisher={Springer Science \& Business Media}
}

@incollection{benaim2006dynamics,
  title={Dynamics of stochastic approximation algorithms},
  author={Bena{\"\i}m, Michel},
  booktitle={Seminaire de probabilites XXXIII},
  pages={1--68},
  year={2006},
  publisher={Springer}
}

@article{mertikopoulos2024unified,
  title={A unified stochastic approximation framework for learning in games},
  author={Mertikopoulos, Panayotis and Hsieh, Ya-Ping and Cevher, Volkan},
  journal={Mathematical Programming},
  volume={203},
  number={1},
  pages={559--609},
  year={2024},
  publisher={Springer}
}

@inproceedings{dvurechensky2024analysis,
  title={Analysis of Kernel Mirror Prox for Measure Optimization},
  author={Dvurechensky, Pavel and Zhu, Jia-Jie},
  booktitle={International Conference on Artificial Intelligence and Statistics},
  pages={2350--2358},
  year={2024},
  organization={PMLR}
}

@inproceedings{karimi2024sinkhorn,
  title={Sinkhorn Flow as Mirror Flow: A Continuous-Time Framework for Generalizing the Sinkhorn Algorithm},
  author={Karimi, Mohammad Reza and Hsieh, Ya-Ping and Krause, Andreas},
  booktitle={International Conference on Artificial Intelligence and Statistics},
  pages={4186--4194},
  year={2024},
  organization={PMLR}
}

@article{leger2021gradient,
  title={A gradient descent perspective on Sinkhorn},
  author={L{\'e}ger, Flavien},
  journal={Applied Mathematics \& Optimization},
  volume={84},
  number={2},
  pages={1843--1855},
  year={2021},
  publisher={Springer}
}

@inproceedings{hessel2021clipscore,
  title={CLIPScore: A Reference-free Evaluation Metric for Image Captioning},
  author={Hessel, Jack and Holtzman, Ari and Forbes, Maxwell and Le Bras, Ronan and Choi, Yejin},
  booktitle={Proceedings of the 2021 Conference on Empirical Methods in Natural Language Processing},
  pages={7514--7528},
  year={2021}
}

@misc{rombach2021highresolution,
      title={High-Resolution Image Synthesis with Latent Diffusion Models}, 
      author={Robin Rombach and Andreas Blattmann and Dominik Lorenz and Patrick Esser and Björn Ommer},
      year={2021},
      eprint={2112.10752},
      archivePrefix={arXiv},
      primaryClass={cs.CV}
}

@article{zhao2024scores,
  title={Scores as Actions: a framework of fine-tuning diffusion models by continuous-time reinforcement learning},
  author={Zhao, Hanyang and Chen, Haoxian and Zhang, Ji and Yao, David D and Tang, Wenpin},
  journal={arXiv preprint arXiv:2409.08400},
  year={2024}
}

@article{jia2022policy,
  title={Policy evaluation and temporal-difference learning in continuous time and space: A martingale approach},
  author={Jia, Yanwei and Zhou, Xun Yu},
  journal={Journal of Machine Learning Research},
  volume={23},
  number={154},
  pages={1--55},
  year={2022}
}

@article{wang2020reinforcement,
  title={Reinforcement learning in continuous time and space: A stochastic control approach},
  author={Wang, Haoran and Zariphopoulou, Thaleia and Zhou, Xun Yu},
  journal={Journal of Machine Learning Research},
  volume={21},
  number={198},
  pages={1--34},
  year={2020}
}

@inproceedings{hsieh2019finding,
  title={Finding mixed nash equilibria of generative adversarial networks},
  author={Hsieh, Ya-Ping and Liu, Chen and Cevher, Volkan},
  booktitle={International Conference on Machine Learning},
  pages={2810--2819},
  year={2019},
  organization={PMLR}
}

@article{domingo2024adjoint,
  title={Adjoint matching: Fine-tuning flow and diffusion generative models with memoryless stochastic optimal control},
  author={Domingo-Enrich, Carles and Drozdzal, Michal and Karrer, Brian and Chen, Ricky TQ},
  journal={arXiv preprint arXiv:2409.08861},
  year={2024}
}

@inproceedings{hazan2019maxent,
  title={Provably Efficient Maximum Entropy Exploration},
  author={Hazan, Elad and Kakade, Sham and Singh, Karan and Van Soest, Abby},
  booktitle={International Conference on Machine Learning},
  year={2019},
}

@article{song2019generative,
  title={Generative modeling by estimating gradients of the data distribution},
  author={Song, Yang and Ermon, Stefano},
  journal={Advances in neural information processing systems},
  volume={32},
  year={2019}
}

@article{ho2020denoising,
  title={Denoising diffusion probabilistic models},
  author={Ho, Jonathan and Jain, Ajay and Abbeel, Pieter},
  journal={Advances in neural information processing systems},
  volume={33},
  pages={6840--6851},
  year={2020}
}

@article{corso2022diffdock,
  title={Diffdock: Diffusion steps, twists, and turns for molecular docking},
  author={Corso, Gabriele and St{\"a}rk, Hannes and Jing, Bowen and Barzilay, Regina and Jaakkola, Tommi},
  journal={arXiv preprint arXiv:2210.01776},
  year={2022}
}

@article{chi2023diffusion,
  title={Diffusion policy: Visuomotor policy learning via action diffusion},
  author={Chi, Cheng and Feng, Siyuan and Du, Yilun and Xu, Zhenjia and Cousineau, Eric and Burchfiel, Benjamin and Song, Shuran},
  journal={arXiv preprint arXiv:2303.04137},
  year={2023}
}

@inproceedings{hoogeboom2022equivariant,
  title={Equivariant diffusion for molecule generation in 3d},
  author={Hoogeboom, Emiel and Satorras, V{\i}ctor Garcia and Vignac, Cl{\'e}ment and Welling, Max},
  booktitle={International conference on machine learning},
  pages={8867--8887},
  year={2022},
  organization={PMLR}
}

@article{uehara2024feedback,
  title={Feedback efficient online fine-tuning of diffusion models},
  author={Uehara, Masatoshi and Zhao, Yulai and Black, Kevin and Hajiramezanali, Ehsan and Scalia, Gabriele and Diamant, Nathaniel Lee and Tseng, Alex M and Levine, Sergey and Biancalani, Tommaso},
  journal={arXiv preprint arXiv:2402.16359},
  year={2024}
}

@inproceedings{sohl2015deep,
  title={Deep unsupervised learning using nonequilibrium thermodynamics},
  author={Sohl-Dickstein, Jascha and Weiss, Eric and Maheswaranathan, Niru and Ganguli, Surya},
  booktitle={International conference on machine learning},
  pages={2256--2265},
  year={2015},
  organization={PMLR}
}

@article{lu2018relatively,
  title={Relatively smooth convex optimization by first-order methods, and applications},
  author={Lu, Haihao and Freund, Robert M and Nesterov, Yurii},
  journal={SIAM Journal on Optimization},
  volume={28},
  number={1},
  pages={333--354},
  year={2018},
  publisher={SIAM}
}

@inproceedings{mutny2023active,
  title={Active exploration via experiment design in {M}arkov chains},
  author={Mutny, Mojmir and Janik, Tadeusz and Krause, Andreas},
  booktitle={International Conference on Artificial Intelligence and Statistics},
  year={2023}
}

@inproceedings{mutti2022importance,
  title={The importance of non-markovianity in maximum state entropy exploration},
  author={Mutti, Mirco and De Santi, Riccardo and Restelli, Marcello},
  booktitle={International Conference on Machine Learning},
  pages={16223--16239},
  year={2022},
  organization={PMLR}
}

@article{mutti2022challenging,
  title={Challenging common assumptions in convex reinforcement learning},
  author={Mutti, Mirco and De Santi, Riccardo and De Bartolomeis, Piersilvio and Restelli, Marcello},
  journal={Advances in Neural Information Processing Systems},
  volume={35},
  pages={4489--4502},
  year={2022}
}

@article{mutti2023convex,
  title={Convex reinforcement learning in finite trials},
  author={Mutti, Mirco and De Santi, Riccardo and De Bartolomeis, Piersilvio and Restelli, Marcello},
  journal={Journal of Machine Learning Research},
  volume={24},
  number={250},
  pages={1--42},
  year={2023}
}

@article{de2024geometric,
  title={Geometric active exploration in Markov decision processes: the benefit of abstraction},
  author={De Santi, Riccardo and Joseph, Federico Arangath and Liniger, Noah and Mutti, Mirco and Krause, Andreas},
  journal={arXiv preprint arXiv:2407.13364},
  year={2024}
}

@article{de2024global,
  title={Global reinforcement learning: Beyond linear and convex rewards via submodular semi-gradient methods},
  author={De Santi, Riccardo and Prajapat, Manish and Krause, Andreas},
  journal={arXiv preprint arXiv:2407.09905},
  year={2024}
}

@article{prajapat2023submodular,
  title={Submodular reinforcement learning},
  author={Prajapat, Manish and Mutn{\`y}, Mojm{\'\i}r and Zeilinger, Melanie N and Krause, Andreas},
  journal={arXiv preprint arXiv:2307.13372},
  year={2023}
}

@article{tang2024fine,
  title={Fine-tuning of diffusion models via stochastic control: entropy regularization and beyond},
  author={Tang, Wenpin},
  journal={arXiv preprint arXiv:2403.06279},
  year={2024}
}

@article{uehara2024fine,
  title={Fine-tuning of continuous-time diffusion models as entropy-regularized control},
  author={Uehara, Masatoshi and Zhao, Yulai and Black, Kevin and Hajiramezanali, Ehsan and Scalia, Gabriele and Diamant, Nathaniel Lee and Tseng, Alex M and Biancalani, Tommaso and Levine, Sergey},
  journal={arXiv preprint arXiv:2402.15194},
  year={2024}
}

@article{skreta2024superposition,
  title={The Superposition of Diffusion Models Using the It$\backslash$\^{} o Density Estimator},
  author={Skreta, Marta and Atanackovic, Lazar and Bose, Avishek Joey and Tong, Alexander and Neklyudov, Kirill},
  journal={arXiv preprint arXiv:2412.17762},
  year={2024}
}

@article{nemirovskij1983problem,
  title={Problem complexity and method efficiency in optimization},
  author={Nemirovskij, Arkadij Semenovi{\v{c}} and Yudin, David Borisovich},
  year={1983},
  publisher={Wiley-Interscience}
}

@article{aubin2022mirror,
  title={Mirror descent with relative smoothness in measure spaces, with application to sinkhorn and em},
  author={Aubin-Frankowski, Pierre-Cyril and Korba, Anna and L{\'e}ger, Flavien},
  journal={Advances in Neural Information Processing Systems},
  volume={35},
  pages={17263--17275},
  year={2022}
}

@article{zahavy2021reward,
  title={Reward is enough for convex mdps},
  author={Zahavy, Tom and O'Donoghue, Brendan and Desjardins, Guillaume and Singh, Satinder},
  journal={Advances in Neural Information Processing Systems},
  volume={34},
  pages={25746--25759},
  year={2021}
}

@article{zeni2023mattergen,
  title={Mattergen: a generative model for inorganic materials design},
  author={Zeni, Claudio and Pinsler, Robert and Z{\"u}gner, Daniel and Fowler, Andrew and Horton, Matthew and Fu, Xiang and Shysheya, Sasha and Crabb{\'e}, Jonathan and Sun, Lixin and Smith, Jake and others},
  journal={arXiv preprint arXiv:2312.03687},
  year={2023}
}

@article{bilodeau2022generative,
  title={Generative models for molecular discovery: Recent advances and challenges},
  author={Bilodeau, Camille and Jin, Wengong and Jaakkola, Tommi and Barzilay, Regina and Jensen, Klavs F},
  journal={Wiley Interdisciplinary Reviews: Computational Molecular Science},
  volume={12},
  number={5},
  pages={e1608},
  year={2022},
  publisher={Wiley Online Library}
}

@article{lipman2024flow,
  title={Flow matching guide and code},
  author={Lipman, Yaron and Havasi, Marton and Holderrieth, Peter and Shaul, Neta and Le, Matt and Karrer, Brian and Chen, Ricky TQ and Lopez-Paz, David and Ben-Hamu, Heli and Gat, Itai},
  journal={arXiv preprint arXiv:2412.06264},
  year={2024}
}

@article{farebrother2025temporal,
  title={Temporal Difference Flows},
  author={Farebrother, Jesse and Pirotta, Matteo and Tirinzoni, Andrea and Munos, R{\'e}mi and Lazaric, Alessandro and Touati, Ahmed},
  journal={arXiv preprint arXiv:2503.09817},
  year={2025}
}

@article{lipman2022flow,
  title={Flow matching for generative modeling},
  author={Lipman, Yaron and Chen, Ricky TQ and Ben-Hamu, Heli and Nickel, Maximilian and Le, Matt},
  journal={arXiv preprint arXiv:2210.02747},
  year={2022}
}

@article{liu2022flow,
  title={Flow straight and fast: Learning to generate and transfer data with rectified flow},
  author={Liu, Xingchao and Gong, Chengyue and Liu, Qiang},
  journal={arXiv preprint arXiv:2209.03003},
  year={2022}
}

@article{albergo2022building,
  title={Building normalizing flows with stochastic interpolants},
  author={Albergo, Michael S and Vanden-Eijnden, Eric},
  journal={arXiv preprint arXiv:2209.15571},
  year={2022}
}

@article{treven2023efficient,
  title={Efficient exploration in continuous-time model-based reinforcement learning},
  author={Treven, Lenart and H{\"u}botter, Jonas and Sukhija, Bhavya and Dorfler, Florian and Krause, Andreas},
  journal={Advances in Neural Information Processing Systems},
  volume={36},
  pages={42119--42147},
  year={2023}
}

@inproceedings{zhang2020variational,
 author = {Zhang, Junyu and Koppel, Alec and Bedi, Amrit Singh and Szepesvari, Csaba and Wang, Mengdi},
 booktitle = {Advances in Neural Information Processing Systems},
 editor = {H. Larochelle and M. Ranzato and R. Hadsell and M.F. Balcan and H. Lin},
 pages = {4572--4583},
 publisher = {Curran Associates, Inc.},
 title = {Variational Policy Gradient Method for Reinforcement Learning with General Utilities},
 url = {https://proceedings.neurips.cc/paper_files/paper/2020/file/30ee748d38e21392de740e2f9dc686b6-Paper.pdf},
 volume = {33},
 year = {2020}
}

@article{li2016renyi,
  title={R{\'e}nyi divergence variational inference},
  author={Li, Yingzhen and Turner, Richard E},
  journal={Advances in neural information processing systems},
  volume={29},
  year={2016}
}

@article{pandey2024heavy,
  title={Heavy-tailed diffusion models},
  author={Pandey, Kushagra and Pathak, Jaideep and Xu, Yilun and Mandt, Stephan and Pritchard, Michael and Vahdat, Arash and Mardani, Morteza},
  journal={arXiv preprint arXiv:2410.14171},
  year={2024}
}

@inproceedings{de2025provable,
  title={Provable Maximum Entropy Manifold Exploration via Diffusion Models},
  author={De Santi, Riccardo and Vlastelica, Marin and Hsieh, Ya-Ping and Shen, Zebang and He, Niao and Krause, Andreas},
  booktitle={International Conference on Machine Learning},
  year={2025},
}

@article{decruyenaere2024debiasing,
  title={Debiasing Synthetic Data Generated by Deep Generative Models},
  author={Decruyenaere, Alexander and Dehaene, Heidelinde and Rabaey, Paloma and Decruyenaere, Johan and Polet, Christiaan and Demeester, Thomas and Vansteelandt, Stijn},
  journal={Advances in Neural Information Processing Systems},
  volume={37},
  pages={41539--41576},
  year={2024}
}

@inproceedings{barakat2023reinforcement,
  title={Reinforcement learning with general utilities: Simpler variance reduction and large state-action space},
  author={Barakat, Anas and Fatkhullin, Ilyas and He, Niao},
  booktitle={International Conference on Machine Learning},
  pages={1753--1800},
  year={2023},
  organization={PMLR}
}

@article{dunn2024mixed,
  title={Mixed continuous and categorical flow matching for 3d de novo molecule generation},
  author={Dunn, Ian and Koes, David Ryan},
  journal={ArXiv},
  pages={arXiv--2404},
  year={2024}
}

@article{friede2024dxtb,
  title={dxtb—An efficient and fully differentiable framework for extended tight-binding},
  author={Friede, Marvin and H{\"o}lzer, Christian and Ehlert, Sebastian and Grimme, Stefan},
  journal={The Journal of Chemical Physics},
  volume={161},
  number={6},
  year={2024},
  publisher={AIP Publishing}
}

@article{ramakrishnan2014quantum,
  title={Quantum chemistry structures and properties of 134 kilo molecules},
  author={Ramakrishnan, Raghunathan and Dral, Pavlo O and Rupp, Matthias and Von Lilienfeld, O Anatole},
  journal={Scientific data},
  volume={1},
  number={1},
  pages={1--7},
  year={2014},
  publisher={Nature Publishing Group}
}

@article{rockafellar2002conditional,
  title={Conditional value-at-risk for general loss distributions},
  author={Rockafellar, R Tyrrell and Uryasev, Stanislav},
  journal={Journal of banking \& finance},
  volume={26},
  number={7},
  pages={1443--1471},
  year={2002},
  publisher={Elsevier}
}

@misc{arjovsky2017wassersteingan,
      title={Wasserstein GAN}, 
      author={Martin Arjovsky and Soumith Chintala and Léon Bottou},
      year={2017},
      eprint={1701.07875},
      archivePrefix={arXiv},
      primaryClass={stat.ML},
      url={https://arxiv.org/abs/1701.07875}, 
}

@article{friedman2022vendi,
  title={The vendi score: A diversity evaluation metric for machine learning},
  author={Friedman, Dan and Dieng, Adji Bousso},
  journal={arXiv preprint arXiv:2210.02410},
  year={2022}
}

@article{chaloner1995bayesian,
  title={Bayesian experimental design: A review},
  author={Chaloner, Kathryn and Verdinelli, Isabella},
  journal={Statistical Science},
  pages={273--304},
  year={1995},
  publisher={JSTOR}
}

@book{pukelsheim2006optimal,
  title={Optimal design of experiments},
  author={Pukelsheim, Friedrich},
  year={2006},
  publisher={SIAM}
}

@phdthesis{mutny2024modern,
  title={Modern Adaptive Experiment Design: Machine Learning Perspective},
  author={Mutn{\`y}, Mojm{\'\i}r},
  year={2024},
  school={ETH Zurich}
}

@article{muandet2017kernel,
  title={Kernel mean embedding of distributions: A review and beyond},
  author={Muandet, Krikamol and Fukumizu, Kenji and Sriperumbudur, Bharath and Sch{\"o}lkopf, Bernhard and others},
  journal={Foundations and Trends{\textregistered} in Machine Learning},
  volume={10},
  number={1-2},
  pages={1--141},
  year={2017},
  publisher={Now Publishers, Inc.}
}
\bibliographystyle{plain}


\newpage
\appendix
\tableofcontents
\newpage
\addtocontents{toc}{\protect\setcounter{tocdepth}{2}}

\section{Functionals and Derivation of Gradients of First-order Variations}
\label{sec:app_functionals_details}
\subsection{Overview of utilities and divergences in Table \ref{table:list_functionals}}
In the following, we report the missing details for the functionals presented within Table \ref{table:list_functionals}, and discuss some possible applications.

\paragraph{Manifold Exploration and Generative Model De-biasing}
As mentioned within Sec. \ref{sec:problem_setting}, maximization of the entropy functional as been recently introduced as a principled objective for manifold exploration \citep{de2025provable}. Moreover, we wish to point out that it can be interpreted also from the viewpoint of de-biasing a prior generative model to re-distribute more uniformly its density while preserving a certain notion of support, \eg via sufficient KL-divergence regularization.

\paragraph{Risk-averse and Novelty-seeking reward maximization} A definition of $q_\beta^r$ can be found below, explanations of these utilities can be found in Sec. \ref{sec:introduction}, and experimental illustrative examples are provided in Sec. \ref{sec:experiments}.

\paragraph{Optimal Experiment Design} 
The task of Optimal Experimental Design (OED) \cite{chaloner1995bayesian} involves choosing a sequence of experiments so as to minimize some uncertainty metric for an unknown \emph{quantity of interest} $f: \X \to \R$, where $\X$ is the set of all possible experiments.  From a probabilistic standpoint, an optimal design may be viewed as a probability distribution over $\X$, prescribing how frequently each experiment should be performed to achieve maximal reduction in uncertainty about $f$ \cite{pukelsheim2006optimal}. This problem has been recently studied in the case where $f$ is an element of a reproducing kernel Hilbert space (RKHS), \ie  $f \in \mathcal{H}_k$, induced by a known kernel $k(x, x') = \Phi(x)^\top \Phi(x')$ where $x, x' \in \X$~\citep{mutny2024modern}. Given this setting, one might aim to acquire information about $f$ according to different \emph{criteria} captured by the scalarization function $s(\cdot)$~\citep{mutny2023active}. In particular, in Table \ref{table:list_functionals}, we report three illustrative choices for $s$:
\begin{itemize}
    \item D-design: $\log\det(\cdot)$ (Information)
    \item A-design: $- \mathrm{Tr}(\cdot)$ (Parameter error)
    \item E-design: $\lambda_{max}(\cdot)$ (Worst projection error)
\end{itemize}
as reported in previous work~\citep[Table 1][]{mutny2023active}.

\paragraph{Diverse Mode Discovery}
This objective corresponds to a re-interpretation of the Diverse Skill Discovery objective introduced in the context of Reinforcement Learning~\citep{zahavy2021reward}. Consider the case where it is given a discrete and finite set $\mathcal{S}$ of symbols interpretable as latent variables, which can be leveraged to (exactly or approximately) perform conditional generation. This objective captures the task of assuring maximal diversity, in terms of KL divergence between the different conditional components, represented as $p^{\pi, k}$ with $k \in \mathcal{S}$.

\paragraph{Log-barrier constrained generation}
This formulation can be found within the General Utilities RL literature~\citep{zhang2020variational}. In particular, here we show the case where constraints are enforced via a log-barrier function, namely $\log(\cdot)$. Nonetheless, the functional presented in Table \ref{table:list_functionals} remains meaningful for general penalty functions.

\paragraph{Optimal transport distances}
OT distances within Table \ref{table:list_functionals} and their relative notation are introduced below in the context of their first variation computation.

\paragraph{Maximum Mean Discrepancy}
Here $k$ denotes a positive-definite kernel, which measures similarity between two points in sample space. Moreover, $\mu_p$ denotes a kernel mean embedding of distribution $p$~\citep{muandet2017kernel}. In terms of applications, choosing a proper kernel $k$ could render possible to preserve specific structure of the initial pre-trained model that would be otherwise lost via KL regularization.

\subsection{A brief tutorial on first variation derivation}
In this work, we focus on the functionals that are Fréchet differentiable: Let $V$ be a normed spaces. Consider a functional $F:V \rightarrow \R$.
    There exists a linear operator $A : V \rightarrow \R$ such that the following limit holds
    \begin{equation} \label{eqn_frechet_derivative_def}
        \lim_{\|h\|_V\rightarrow 0} \frac{|F(f + h) - F(f) - A[h]|}{\|h\|_V} = 0.
    \end{equation}
    We further assume that $V$ admits certain structure such that every element in its dual space (the space of bounded linear operator on $V$) admits some compact representation. For example, when $V$ is the set of compact-supported continuous bounded functions, there exists a unique positive Borel measure $\mu$ with the same support, which can be identified as the linear functional.
    We denote this element as $\delta F[f]$ such that $\langle \delta F[f], h\rangle = A[h]$. Sometimes we also denote it as $\frac{\delta F}{\delta f}$. We will refer to $\delta F[f]$ as the first-order variation of $F$ at $f$.

In this section, we briefly review strategies for deriving the first-order variation of two broad classes of functionals: those defined in closed form with respect to the density (e.g., expectation and entropy) and those defined via variational formulations (e.g., CVaR, Wasserstein distance, and MMD).
\begin{itemize}[leftmargin=*]
    \item \textbf{Category 1: Functional defined in a closed form w.r.t. the density.} For this class of functionals, the first-order variations can typically be computed using its definition and chain rule.

    With definition (\ref{eqn_frechet_derivative_def}) in mind, we can try to calculate the first-order variation of the mean functional.
    Consider a continuous and bounded function $r: \R^d \rightarrow \R$ and a probability measure $\mu$ on $\R^d$. Consider the functional $F(\mu) = \int r(x) \mu(x) dx$. We have
    \begin{equation}
        |F(\mu + \delta \mu) - F(\mu) - \langle r,  \delta\mu \rangle| = 0.
    \end{equation}
    We therefore obtain $\delta F[\mu] = r$ for all $\mu$.
    We will compute the first-order variations for other functionals in the next subsection.    
    \item \textbf{Category 2: Functionals defined through a variational formulation.} Another important subclass of functionals considered in this paper is the ones defined via a variational problem
    \begin{equation}
        F[f] = \sup_{g \in \Omega} G[f, g],
    \end{equation}
    where $\Omega$ is a set of functions or vectors independent of the choice of $f$, and $g$ is optimized over the set $\Omega$. We will assume that the maximizer $g^*(f)$ that reaches the optimal value for $G[f, \cdot]$ is unique (which is the case for the functionals considered in this project).    
    It is known that one can use the Danskin's theorem (also known as the envelope theorem) to compute
    \begin{equation}
        \frac{\delta F[f]}{\delta f} = \partial_f G[f, g^*(f)],
    \end{equation}
    under the assumption that $F$ is differentiable \citep{milgrom2002envelope}.
\end{itemize}
\subsection{Derivation of gradients of first-order variation for functionals in Table \ref{table:list_functionals}}

\begin{table*}[t]
\setlength{\tabcolsep}{4pt}
\renewcommand{\arraystretch}{2.5}
\vspace{-1mm}
\begin{sc}
\resizebox{\textwidth}{!}{%
\begin{tabular}{c c c c c}
\toprule
\multirow{2}{*}{Application} &
\multirow{2}{*}{Functional $\F$ / $\D$} &
\multirow{2}{*}{\shortstack[c]{First-order Variation}} &
\multicolumn{2}{c}{Density Control} \\ \cmidrule(lr){4-5}
 & & & convex & general \\
\midrule
\rowcolor{lightblueII} Reward optimization~\citep{domingo2024adjoint, uehara2024fine} &
  $\mathbb{E}_{x \sim p^\pi}[r(x)]$ &
  $r$ & \tick & \tick \\
\rowcolor{lightblue}%
\multirow{2}{*}{\shortstack[c]{Manifold Exploration\\[0.1em]Gen.\ model de-biasing}} &
  \multirow{2}{*}{$\entropy(p^\pi) \coloneqq -\EV_{x \sim p^\pi}[\log p^\pi(x)]$} &
  \multirow{2}{*}{$-1-\log p^\pi$} & \multirow{2}{*}{\tick} & \multirow{2}{*}{\tick} \\[1.45em]
\rowcolor{lightblueII} \multirow{2}{*}{Risk-averse optimization} &
  $\mathrm{CVaR}^r_{\beta}(p^\pi) \coloneqq \EV_{x\sim p^\pi}[r(x) \mid r(x) \leq \mathrm{q}^r_{\beta}(p^\pi)]$ & $\beta \min\{r(x) - q^r_\beta(p^\pi), 0\}$ & \tick & \tick \\[0.6ex]
\rowcolor{lightblueII} &  $\mathbb{E}_{x \sim p^\pi}[r(x)] - \mathbb{V}\mathrm{ar}(p^\pi)$ & $r(x) - \left(r(x)^2 - 2\mathbb{E}_{x\sim p^\pi}[r(x)] r(x)\right)$ & \cross & \tick \\
\rowcolor{lightblue}%
Risk-seeking optimization &
  $\text{SQ}^r_{\beta}(p^\pi) \coloneqq \EV_{x\sim p^\pi}[r(x) \mid r(x) \geq q^r_\beta(p^\pi)]$ & $(1-\beta)\max\{r(x) - q^r_\beta(p^\pi), 0\}$ & \cross & \tick \\
\rowcolor{lightblueII} \multirow{2}{*}{Optimal Experiment Design} &
  $\mathrm{s}(\EV_{x\sim p^\pi}[\Phi(x)\Phi(x)^\top - \lambda \mathbb{I}])$ & 
  see \cref{eqn_first_order_variation_log_det} & \multirow{2}{*}{\tick} & \multirow{2}{*}{\tick} \\[0.6ex]
\rowcolor{lightblueII} &  $\mathrm{s}(\cdot) \in \{\log \det(\cdot), -\mathrm{Tr}(\cdot)^{-1}, -\lambda_{max}(\cdot) \}$  &  &  &  \\
\rowcolor{lightblue} \rule[-1.5ex]{0pt}{4.0ex}
Diverse modes discovery &
  $- \EV_z[D_{KL}(p^{\pi, z} \| \EV_k p^{\pi, k})]$ & See \cref{eqn_first_order_variation_mode_discovery} & \cross & \tick \\
\rowcolor{lightblueII} Log-Barrier Constrained Generation &
  $\mathbb{E}_{x \sim p^\pi}[r(x)] - \beta \log \left(\langle p^\pi, c \rangle - C \right)$ &
  See \Cref{eqn_first_order_variation_log_barrier} & \tick & \tick \\
\midrule
\rowcolor{lightorange} \rule[-1.5ex]{0pt}{4.0ex} Kullback–Leibler divergence &
  $D_{KL}(p^\pi \, \| \, p^{pre}) = \int p^\pi(x)\log \frac{p^\pi(x)}{p^{pre}(x)} \, dx$ & $1 + \log p^\pi -\log p^{pre}$ & \tick & \tick \\
\rowcolor{lightorangeII} \rule[-1.5ex]{0pt}{4.0ex} Rényi divergences &
  $D_\beta(p^\pi \, \| \, p^{pre}) \coloneqq \frac{1}{\beta -1}\log \int (p^\pi(x))^\beta (p^{pre}(x))^{1-\beta} \, dx$ &
  $\frac{\beta}{\beta-1} \left(\int \left(\frac{p}{q}\right)^{\beta} dq(x)\right)^{-1} \left(\frac{p}{q}\right)^{\beta-1}$ & \tick & \tick \\
\rowcolor{lightorange} \rule[-1.5ex]{0pt}{4.0ex} Optimal Transport distances &
  $W_p(p^\pi \, \| \, p^{pre}) \coloneqq \inf_{\gamma \in \Gamma(p^\pi, p^{pre})} \EV_{(x,y) \sim \gamma}[d(x,y)^p]^{\frac{1}{p}}$ &
  see \cref{eqn_first_order_variation_Wp} & \tick & \tick \\
\rowcolor{lightorangeII} \rule[-1.5ex]{0pt}{4.0ex} Maximum Mean Discrepancy &
  $\mathrm{MMD}_k(p^\pi, p^{pre}) \coloneqq  \| \mu_{p^\pi} - \mu_{p^{pre}}\|$, $\mu_p \coloneqq \EV_{x\sim p}[k(x, \cdot)]$ &
  $\argmax_{\phi \in \mathcal{H}} \langle\phi, p^\pi - p^{pre}\rangle$ & \tick & \tick \\
\bottomrule
\end{tabular}
}
\end{sc}
\caption{\looseness-1 Examples of practically relevant utilities $\F$ (blue) and divergences $\D$ (orange), and their first-order variations.}
\end{table*}
\begin{itemize}[leftmargin=*]
    \item \textbf{Risk-Averse Optimization (Category 2)}
    Recall that $q_\beta^r(p^\pi) = \sup\{v\in\R \vert F_Z(v) \leq \beta\}$, where the random variable $Z$ is defined as $Z = r(x)$ with $x\sim p^\pi(x)$.
    From \citep{rockafellar2000optimization}, we have
    \begin{equation*}
        \mathrm{CVaR}_\beta^r(p^\pi) = \mathbb{E}[r(x) \vert r(x) \leq q_\beta^r(p^\pi)] = \beta \inf_\zeta \left\{\zeta + \frac{1}{\beta}\mathbb{E}\left[\min\{r(x) - \zeta, 0 \}\right]\right\}.
    \end{equation*}
    Moreover, we have $\zeta^*$ that solves the above optimization problem is exactly $\zeta^* = q_\beta^r(p^\pi)$.
    By Danskin's theorem, one has (in a weak sense)
    \begin{equation}
        \frac{\delta \mathrm{CVaR}_\beta^r(p^\pi)}{\delta p^\pi} = \beta \min\{r(x) - q_\beta^r(p^\pi), 0\}.
    \end{equation}
    \item \textbf{Risk-Seeking Optimization (Category 2)}
    Recall that $q_\beta^r(p^\pi) = \sup\{v\in\R \vert F_Z(v) \leq \beta\}$, where the random variable $Z$ is defined as $Z = r(x)$ with $x\sim p^\pi(x)$.
    From \citep{rockafellar2000optimization}, we have
    \begin{equation*}
        \mathrm{SQ}_\beta^r(p^\pi) = \mathbb{E}[r(x) \vert r(x) \geq q_\beta^r(p^\pi)] = (1-\beta) \inf_\zeta \left\{\zeta + \frac{1}{1-\beta}\mathbb{E}\left[\max\{r(x)-\zeta, 0 \}\right]\right\}.
    \end{equation*}
    Moreover, we have $\zeta^*$ that solves the above optimization problem is exactly $\zeta^* = q_\beta^r(p^\pi)$.
    By Danskin's theorem, one has (in a weak sense)
    \begin{equation}
        \frac{\delta \mathrm{SQ}_\beta^r(p^\pi)}{\delta p^\pi} = (1-\beta)\max\{r(x) - q_\beta^r(p^\pi), 0\}.
    \end{equation}
    \item \textbf{Rényi Divergence (Category 1)} Recall the definition of Rényi Divergence
    \begin{equation}
        D_\beta(p \| q) = \frac{1}{\beta-1} \log \int \left(\frac{p}{q}\right)^{\beta} dq(x).
    \end{equation}
    We ignore higher-order terms like $O((\delta p)^2)$.
    \begin{align}
        D_\beta(p + \delta p \| q) - D_\beta(p \| q) =&\ \frac{1}{\beta-1} \log \frac{\int \left(\frac{p+\delta p}{q}\right)^{\beta} dq(x)}{\int \left(\frac{p}{q}\right)^{\beta} dq(x)} \\
        =&\ \frac{1}{\beta-1}  \log \frac{\int \left(\frac{p}{q}\right)^{\beta} + \beta \left(\frac{p}{q}\right)^{\beta-1}\frac{\delta p}{q}  dq(x)}{\int \left(\frac{p}{q}\right)^{\beta} dq(x)} \\
        =&\ \frac{1}{\beta-1} \log 1 + \frac{\int  \beta \left(\frac{p}{q}\right)^{\beta-1}\frac{\delta p}{q}  dq(x)}{\int \left(\frac{p}{q}\right)^{\beta} dq(x)} \\
        =&\ \frac{1}{\beta-1} \frac{\int  \beta \left(\frac{p}{q}\right)^{\beta-1}\frac{\delta p}{q}  dq(x)}{\int \left(\frac{p}{q}\right)^{\beta} dq(x)}
    \end{align}
    \begin{align}
        \frac{\delta}{\delta p} R_\beta(p, q) = \frac{\beta}{\beta-1} \left(\int \left(\frac{p}{q}\right)^{\beta} dq(x)\right)^{-1} \left(\frac{p}{q}\right)^{\beta-1}
    \end{align}
    \item \textbf{Optimal transport and Wasserstein-p distance (Category 2)}
    Consider the optimal transport problem
    \begin{equation}
        \mathrm{OT}_c(u, v) = \inf_{\gamma}\left\{\int\int c(x, y) d\gamma(x, y): \int \gamma(x, y) dx = u(y), \int \gamma(x, y) d y = v(x)\right\}
    \end{equation}
    where 
    \begin{equation*}
        \Gamma = \left\{ \gamma : \int \gamma(x, y) dx = u(y), \int \gamma(x, y) d y = v(x) \right\}
    \end{equation*}
    It admits the following equivalent dual formulation
    \begin{align}
        \mathrm{OT}_c(u, v) = \sup_{f, g} \left\{\int f du + \int g dv: f(x) + g(y) \leq c(x, y) \right\} 
    \end{align}
    By taking $c(x, y) = \|x-y\|^p$, we recover $\mathrm{OT}_c(u, v) = W_p(u, v)^p$.
    Let $f^*$ and $g^*$ be the solution to the above dual optimization problem.
    From the Danskin's theorem, we have
    \begin{equation} \label{eqn_first_order_variation_Wp}
        \frac{\delta}{\delta u} W_p(u, v)^p = f^*.
    \end{equation}
    In the special case of $p=1$, we know that $g^* = - f^*$ (note that the constraint can be equivalently written as $\|\nabla f\| \leq 1$), in which case $f^*$ is typically known as the critic in the WGAN framework.
    \item  \textbf{Optimal Experiment Design. (Category 1)} We take $\mathrm{s}(M) = \log \det (M)$ as example. 
    By chain rule, we have
    \begin{equation} \label{eqn_first_order_variation_log_det}
        \delta F[p^\pi] = \mathrm{Tr}\left[\left(\EV_{x\sim p^\pi}[\Phi(x)\Phi(x)^\top - \lambda \mathbb{I}]\right)^{-1} \left(\Phi(x)\Phi(x)^\top - \lambda \mathbb{I}\right) \right].
    \end{equation}
    \item  \textbf{Log-Barrier Constrained Generation. (Category 1)} By chain rule, we obtain
    \begin{equation} \label{eqn_first_order_variation_log_barrier}
        \delta F[p^\pi] = r -  \frac{\beta c}{\langle p^\pi, c\rangle - C}.
    \end{equation}
    \item \textbf{Diverse modes discovery. (Category 1)} By chain rule, we obtain
    \begin{align}
        \frac{\delta F}{\delta p^{\pi, z}} =&\ -\frac{\delta}{\delta p^{\pi, z}}\mathbb{E}_z\left[\int p^{\pi, z} \log p^{\pi, z} dx - \int p^{\pi, z} \log\left(\mathbb{E}_k[p^{\pi, k}]\right) dx\right] \notag \\
        =&\ -\mathbb{E}_z\left[\frac{\delta}{\delta p^{\pi, z}}\left(\int p^{\pi, z} \log p^{\pi, z} dx\right) - \frac{\delta}{\delta p^{\pi, z}} \left(\int p^{\pi, z} \log\left(\mathbb{E}_k[p^{\pi, k}]\right) dx\right)\right] \notag \\
        =&\ -\mathbb{E}_z\left[\log p^{\pi, z} + 1 - \log\left(\mathbb{E}_k[p^{\pi, k}]\right) - \frac{p^{\pi, z}}{\mathbb{E}_k[p^{\pi, k}]} \right] \label{eqn_first_order_variation_mode_discovery}
    \end{align}
\end{itemize}
\begin{itemize}[leftmargin=*]
    \item \textbf{Entropy. (Category 1)} As a first example, consider the entropy functional $\F(p) = -\int p \log p , dx$. By the definition of the first-order variation, we have $\frac{\delta \F}{\delta p}(p) = -1 - \log p$, and therefore $\nabla \frac{\delta \F}{\delta p}(p) = -\nabla \log p$. This gradient term can be effectively estimated using standard score approximations; see \cite{de2025provable}.
\end{itemize}
\newpage

\section{Proof for Theorem \ref{theorem:convex_case_convergence}}
\label{sec:app-theory1}
\convexCaseConvergence*
\begin{proof}
We prove this result using the framework of relative smoothness and relative strong convexity introduced in \Cref{sec:theory}.

The analysis is based on the classical mirror descent framework under relative properties~\citep{lu2018relatively}. For notational simplicity, we let $\mu_k \coloneqq p_T^{\pi_k}$, and fix an arbitrary reference density $\mu \in \mP(\Omega_{\mathrm{pre}})$. To better align the notation of our theory with existing literature, we will proceed with the \emph{convex} functional $\tilde{\G} \defeq -\G$ below.

We begin by showing the following inequality:
\begin{align}
    \tilde{\G}(\mu_k) 
    &\leq \tilde{\G}(\mu_{k-1}) + \langle \delta \tilde{\G}(\mu_{k-1}), \mu_k - \mu_{k-1} \rangle + L D_\Q(\mu_k, \mu_{k-1}) \\
    &\leq \tilde{\G}(\mu_{k-1}) + \langle \delta \tilde{\G}(\mu_{k-1}), \mu - \mu_{k-1} \rangle + L D_\Q(\mu, \mu_{k-1}) - L D_\Q(\mu, \mu_k).
\end{align}
The first inequality follows from the $L$-smoothness of $\G$ relative to $\Q$ as defined in \Cref{definition:relative_properties}. The second inequality uses the three-point inequality of the Bregman divergence~\citep[Lemma 3.1]{lu2018relatively} with $\phi(\mu) = \frac{1}{L} \langle \delta \G(\mu_{k-1}), \mu - \mu_{k-1} \rangle$, $z = \mu_{k-1}$, and $z^+ = \mu_k$.

Next, using the $l$-strong concavity of $\G$ relative to $\Q$, again from \Cref{definition:relative_properties}, we obtain:
\begin{equation}
    \tilde{\G}(\mu_k) \leq \tilde{\G}(\mu) + (L - l) D_\Q(\mu, \mu_{k-1}) - L D_\Q(\mu, \mu_k).
\end{equation}

By recursively applying the above inequality and using the monotonicity of $\G(\mu_k)$ along with the non-negativity of the Bregman divergence, we obtain~\citep{lu2018relatively}:
\begin{align}
    \sum_{k=1}^K \left(\frac{L}{L - l}\right)^k \left(\tilde{\G}(\mu_k) - \tilde{\G}(\mu) \right)
    \leq L D_\Q(\mu, \mu_0) - L \left( \frac{L}{L - l} \right)^K D_\Q(\mu, \mu_K)
    \leq L D_\Q(\mu, \mu_0).
\end{align}

Letting
\begin{equation}
    \frac{1}{C_K} \coloneqq \sum_{k=1}^K \left( \frac{L}{L - l} \right)^k,
\end{equation}
and rearranging terms, we arrive at the convergence rate:
\begin{equation}
   \tilde{\G}(\mu_K) - \tilde{\G}(\mu) \leq C_K L D_\Q(\mu, \mu_0) = \frac{l D_\Q(\mu, \mu_0)}{\left(1 + \frac{l}{L - l}\right)^K - 1}. \label{eq:last_eq_proof}
\end{equation}

Finally, the convergence rate stated in the theorem follows by observing that $\left(1 + \frac{l}{L - l}\right)^K \geq 1 + \frac{K l}{L - l}$.

\end{proof}
\newpage

\section{Proof for Theorem \ref{theorem:general_case_convergence}}
\label{sec:app-theory2}

\newtheorem{informalassumption}{Informal Assumption}
To establish our main convergence result, we introduce two additional technical assumptions that are satisfied in virtually all practical settings:

\begin{assumption}[Support Compatibility]
\label{asm:support}
We assume that the support of $p_T^{\pi_k}$ is contained in a fixed compact domain $\tilde{\Omega}$ for all $k$, and that for some $j$, we have $\text{supp}(p_j^{\pi_k}) = \tilde{\Omega}$.
\end{assumption}\vspace{-2mm}

\begin{assumption}[Precompactness]
\label{asm:precompact}
The sequence $\{\delta \entropy(p_T^{\pi_k})\}_k$ is precompact in the topology induced by the $L_\infty$ norm.
\end{assumption}

We are now ready to present the full proof. For the reader's convenience, we restate the theorem:

\generalCaseConvergence*

\begin{proof}
\newcommand{\drm}{\mathrm{d}}
We divide the proof into several key steps for ease of reading.

\para{Continuous-Time Mirror Flow}
The main idea of our proof is to relate the discrete iterates $\{p_T^k\}_{k\in \mathbb{N}}$ produced by \cref{alg:fdc_algorithm} to a continuous-time mirror flow.

Define the initial dual variable as
\[
\dual_0 = \delta \entropy(p_{\mathrm{pre}}) = -\log p_{\mathrm{pre}},
\]
and consider the gradient flow
\begin{equation}\label{eq:MF} \tag{MF}
\begin{cases}
\dot{\dual}_t = \delta \G(p_t),\\
p_t = \delta(-\entropy)^\star(\dual_t),
\end{cases}
\end{equation}
where $(-\entropy)^\star(\dual) = \log \int_\Omega e^{\dual}$ is the Fenchel dual of the negative entropy functional \cite{hsieh2019finding, hiriart2004fundamentals}. This defines the deterministic mirror flow associated with $\G$.

\para{Continuous-Time Interpolation of Iterates}
To connect the discrete algorithm with \eqref{eq:MF}, we construct a continuous-time interpolation of the dual iterates
\(\curr = \delta \entropy(p_T^{\pi_k})\). Define the effective time
\[
\curr[\efftime] = \sum_{r=0}^{k} \alpha_r,
\]
and let the interpolated process $\apt{t}$ be
\begin{equation}\label{eq:interpolation} \tag{Int}
\apt{t} = \curr + \frac{t - \curr[\efftime]}{\curr[\efftime+1] - \curr[\efftime]} (\curr[\efftime+1] - \curr).
\end{equation}

Intuitively, our convergence result follows if two conditions hold:

\begin{informalassumption}[Closeness to Continuous-Time Flow]
\label{iasm:dis2cont}
The interpolated process $\apt{t}$ asymptotically follows the dynamics of \eqref{eq:MF} as $k \to \infty$.
\end{informalassumption}

\begin{informalassumption}[Convergence of the Flow]
\label{iasm:sol}
The trajectories of \eqref{eq:MF} converge to a stationary point of $\G$.
\end{informalassumption}

To formalize this, we invoke the stochastic approximation framework of \cite{benaim2006dynamics}. Let $\points$ be the space of integrable functions on $\Omega$, and let $\flowmap$ denote the flow of \eqref{eq:MF}. We define:

\begin{definition}[Asymptotic Pseudotrajectory (APT)]
\label{def:APT}
We say $\apt{t}$ is an \ac{APT} of \eqref{eq:MF} if for all $T > 0$,
$$
\lim_{t\to\infty} \sup_{0 \leq h \leq T} \| \apt{t+h} - \flowmap_h(\apt{t}) \|_\infty = 0.
$$
\end{definition}

If $\apt{t}$ is a precompact \ac{APT}, then \cite{benaim2006dynamics} show:

\begin{theorem}[APT Limit Set Theorem]
\label{thm:apt2ict}
Let $\apt{t}$ be a precompact \ac{APT} for the flow \eqref{eq:MF}. Then, almost surely, the limit set of $\apt{t}$ is contained in the set of \ac{ICT} points of \eqref{eq:MF}.
\end{theorem}

The proof of our result thus follows from two claims:
\begin{enumerate}
    \item The iterates $\{\curr\}$ generate a precompact \ac{APT} under Assumptions \ref{asm:support} and\ref{asm:approximate}.
    \item The \ac{ICT} set of \eqref{eq:MF} consists only of stationary points of $\G$.
\end{enumerate}
The remainder of the proof is devoted to verifying these two claims.

\para{Convergence to Stationary Points}
The second claim holds since the mirror flow admits $\G$ as a strict Lyapunov function, and thus Corollary 6.6 in \cite{benaim2006dynamics} ensures convergence of the APT to the set of stationary points of $\G$, provided that the set of equilibria is countable.

For the first claim, Assumptions \ref{asm:support} and \ref{asm:precompact} ensure that the interpolated process is well-defined and precompact, while Assumption \ref{asm:approximate} allows us to apply standard stochastic approximation arguments \cite{karimi2024sinkhorn}. We conclude the proof by applying \cref{thm:apt2ict}.

\para{Quantitative Approximation to the Mirror Flow}
For the first claim, we invoke the stochastic approximation techniques applied to the dual variables (see, e.g., \cite{benaim2006dynamics, karimi2024sinkhorn}) to obtain the following bound:
\begin{equation}\label{eq:quantitative}
\sup_{0 \le s \le T} \| \apt{t+s} - \flowmap_s(\apt{t}) \| \le C(T) \big[ \Delta(t-1, T+1) + b(T) + \gamma(T) \big],
\end{equation}
where $C(T)$ depends only on $T$, $\Delta(t-1,T+1)$ captures cumulative noise fluctuations, and $b(T),\gamma(T)$ are the bias and step-size terms over the interval. This explicitly bounds the deviation of the interpolated process from the deterministic mirror flow.

Under the noise and bias conditions in Assumption~\ref{asm:approximate}, standard stochastic approximation results \cite{benaim2006dynamics, karimi2024sinkhorn} imply
\[
\lim_{T \to \infty} \Delta(t-1, T+1) = \lim_{T \to \infty} b(T) = 0.
\]
Hence, $\apt{t}$ is an APT of the mirror flow.

\para{Conclusion}
Assuming precompactness of the dual iterates (stated as Assumption~\ref{asm:precompact}), Theorem 5.7 in \cite{benaim2006dynamics} implies that the limit set of $\apt{t}$ is internally chain transitive (ICT) for the mirror flow. Combining the quantitative approximation \eqref{eq:quantitative}, the APT argument, and the limit set characterization, we conclude that the discrete iterates converge to stationary points of $\G$, completing the proof.
\end{proof}

\newpage

\section{Detailed Example of Algorithm Implementation}
\label{sec:alg_implementation}
\subsection{Implementation of \LinearFineTuningSolver}
\looseness -1 To ensure completeness, below we provide pseudocode for one concrete realization of a \LinearFineTuningSolver as in Eq. \eqref{eq:opt_first_variation} using a first-order optimization routine. In particular, we describe exactly the version employed in Sec. \ref{sec:experiments}, which builds on the Adjoint Matching framework~\citep{domingo2024adjoint}, casting linear fine-tuning as a stochastic optimal control problem and tackling it via regression.

Let $u^{pre}$ be the initial, pre-trained vector field, and $u^{finetuned}$ its fine-tuned counterpart. We also use $\bar{\alpha}$ to refer to the accumulated noise schedule from \citep{ho2020denoising} effectively following the flow models notation introduced by Adjoint Mathing~\citep[][Sec. 5.2]{domingo2024adjoint}. The full procedure is in Algorithm \ref{alg:step_implemented}.

\begin{algorithm}[H]
    \caption{\LinearFineTuningSolver (Adjoint Matching \citep{domingo2024adjoint}) based implementation}
    \label{alg:step_implemented}
        \begin{algorithmic}[1]
        \State{\textbf{Input: } $N: $ number of iterations, $u^{pre}: $ pre-trained flow vector field, $\eta$ regularization coefficient as in Eq. \eqref{eq:opt_first_variation}, $h:$ step size, $\nabla f$: reward function gradient, $m$ batch size}
        \State{\textbf{Init:} $u^{finetuned} \coloneqq u^{pre}$ with parameter $\theta$}
        \For{$n=0, 1, 2, \hdots, N-1$}
            \State{Sample $m$ trajectories $\{X_t\}_{t=1}^T$ via memoryless noise schedule~\citep{domingo2024adjoint}, \eg 
            \begin{equation*}
                \text{sample } \epsilon_t \sim \mathcal{N}(0,I), \; X_0 \sim \mathcal{N}(0,I) \text{, then:}
            \end{equation*}
           \begin{equation*}
               X_{t+h} = X_t + h\left(2v_{\theta}^{finetuned}(X_t, t) - \frac{\bar{\alpha}_t}{\alpha_t}X_t\right) + \sqrt{h} \sigma(t) \epsilon_t
           \end{equation*}
        Use reward gradient: $$\Tilde{a}_T = - \frac{1}{\eta} \nabla  f(X_1)$$
        For each trajectory, solve the lean adjoint ODE, see \citep[Eq. 38-39]{domingo2024adjoint}, from $t=1$ to $0$, \eg:
        \begin{equation*}
            \Tilde{a}_{t-h} = \Tilde{a}_t + h\Tilde{a}_t^\top \nabla_{X_t}\left(2 u^{pre}(X_t,t) - \frac{\bar{\alpha}_t}{\alpha_t}X_t \right)
        \end{equation*}
        \looseness -1 Where $X_t$ and $\Tilde{a}_t$ are computed without gradients, \ie $X_t = \texttt{stopgrad}(X_t), \Tilde{a}_t = \texttt{stopgrad}(\Tilde{a}_t)$.
        For each trajectory compute the Adjoint Matching objective \citep[Eq. 37]{domingo2024adjoint}:
        \begin{equation*}
            \mathcal{L}_{\theta} = \sum_{t=0}^{1-h} \| \frac{2}{\sigma(t)}\left( u_{\theta}^{finetuned}(X-t,t) - u^{pre}(X_t,t) \right) + \sigma(t) \Tilde{a}_t \|
        \end{equation*}
        Compute the gradient $\nabla_\theta \mathcal{L}(\theta)$ and update $\theta$.}
        \EndFor
        \State{\textbf{output: } Fine-tuned noise predictor $u_{\theta}^{finetuned}$}
        \end{algorithmic}
\end{algorithm}

\subsection{Discussion: computational complexity and cost of \AlgNameShort}
\AlgNameLong (see Algorithm \ref{alg:fdc_algorithm}) is a sequential fine-tuning scheme, which performs $K$ iterations of a base fine-tuning oracle, as shown in Algorithm \ref{alg:fdc_algorithm}. Typically, as for the case of Adjoint Matching~\citep{domingo2024adjoint}, which is contextualized in Algorithm \ref{alg:step_implemented}, the inner oracle also performs $N$ iterations to solve the classic fine-tuning problem. As a consequence, at first glance, this lead to \AlgNameShort having a computational complexity scaling linearly in $K$ the one of classic fine-tuning. Nonetheless, this does not seem to capture well the practical computational cost. In particular, we wish to point out the two following observations:
\begin{itemize}
    \item As discussed for the molecular design experiment in Sec. \ref{sec:experiments} and further in Appendix \ref{sec:experimental_detail}, the \AlgNameShort scheme might work well even with a very approximate oracle to solve the entropy-regularized control problem at each iteration.
    \item For many real-world problems a very small number of iterations $K$ might be sufficient to approximate the non-linear functional sufficiently well and hence obtain useful fine-tuned models. This is shown in text-to-image bridge design experiment in Sec. \ref{sec:experiments} and in Appendix \ref{sec:experimental_detail}. In this case, merely $K=2$ iterations of \AlgNameShort lead to promising results.
\end{itemize}
\newpage

\section{Experimental Details}
\label{sec:experimental_detail}
\subsection{Used computational resources}
We run all experiments on a single Nvidia H100 GPU. 

\subsection{Experiments in Illustrative Settings}

\paragraph{Shared experimental setup.} For all illustrative experiments we utilize Adjoint Matching (AM)~\citep{domingo2024adjoint} for the entropy-regularized fine-tuning solver in \cref{alg:fdc_algorithm}. Moreover, the stochastic gradient steps within the AM scheme are performed via an Adam optimizer.

\paragraph{Risk-averse reward maximization for better worst-case validity or safety.} 
In this experiment, we execute \AlgNameShort for $K=2$ iterations with a total of $1000$ gradient steps within each iteration, AM solver (within the \AlgNameShort scheme) with learning rate of $2e^{-2}$, $\alpha = 10^9$, and $\eta = 10$. Meanwhile, the AM baseline, is run for $1000$ gradient steps with $\alpha = 0.2857$, and learning rate of $1e^{-5}$. The resulting CVaR is computed via the standard torch quantile method. The values of $\beta$ reported in the main paper effectively refers to the value of $1-\beta$. In the following, we report mean and sample standard deviation of AM and \AlgNameShort over $5$ seeds.
\begin{figure}[h]
\centering
    \begin{tabular}{@{}ll@{}}
            \toprule
               & \quad CVaR$_\beta$              \\
            \midrule
            Pre-trained    $\;$  & \; $256.8 \pm 8.15$ \\
            \AlgNameShortAM  &\; $225.3 \pm 78.9$\\
            \AlgNameShort ($1$ iteration)  &\; $221.1 \pm 73.2$\\
            \AlgNameShort ($2$ iteration)  &\; $90.0 \pm 0.05$\\
            \bottomrule
    \end{tabular}
    \caption{Statistical analysis for CVaR$_\beta$.}
\end{figure}

\paragraph{Novelty-seeking reward maximization for discovery.} 
We run \AlgNameShort for $K=2$ iterations with a total of $1000$ gradient steps within each iteration, AM solver (within the \AlgNameShort scheme) with learning rate of $3e^{-6}$, $\alpha = 10^5$, and $\eta = 0.625$, and $8000$ samples are used to estimate the first variation gradient as explained in Appendix \ref{sec:app_functionals_details}. Meanwhile, the AM baseline, is run for $1000$ gradient steps with $\alpha = 0.666$, and learning rate of $1e^{-5}$. The resulting SQ is computed via the standard torch quantile method. In the following, we report mean and sample standard deviation of AM and \AlgNameShort over $5$ seeds.
\begin{figure}[h]
\centering
    \begin{tabular}{@{}ll@{}}
            \toprule
               & \quad SQ$_\beta$              \\
            \midrule
            Pre-trained    $\;$  & \; $59.6 \pm 7.5$ \\
            \AlgNameShortAM  &\; $56.7 \pm 2.7$\\
            \AlgNameShort ($1$ iteration)  &\; $55.0 \pm 0.04$\\
            \AlgNameShort ($2$ iteration)  &\; $452.5 \pm 250.0$\\
            \bottomrule
    \end{tabular}
    \caption{Statistical analysis for SQ$_\beta$ utility.}
\end{figure} 

\paragraph{Reward maximization regularized via optimal transport distance. } 
Within this experiment, we present two runs of \AlgNameShort, namely \AlgNameShort-A and \AlgNameShort-B, compared against AM.  Both \AlgNameShort-A and \AlgNameShort-B have been run for $K=6$ iterations of \AlgNameShort, with $\alpha = 0.1$, AM oracle learning rate of $1e^{-6}$, $\eta = 6.666$. Both their discriminators to solve the dual OT problem as presented in Appendix \ref{sec:app_functionals_details} and mentioned within Sec. \ref{sec:algorithm}, have been learned via a simple MLP architecture with $800$ gradient steps, by enforcing the $1$-Lip. condition via the standard gradient penalty technique with regularization strength of $\lambda_{GP} = 10.0$ and learning rate of $1e^{-4}$. In particular, \AlgNameShort-A is based on the distance defined, for two $2$-dimensional points $x=(x_1, x_2)$ and $y=(y_1, y_2)$ by:
\begin{equation*}
    d_A(x,y) = \sqrt{(x_1-y_1)^2 + (K (x_2 - y_2))^2}
\end{equation*}
Analogously, \AlgNameShort-B leverages $d_B$ defined as:
\begin{equation*}
    d_A(x,y) = \sqrt{(K(x_1-y_1))^2 + (x_2 - y_2)^2)}
\end{equation*}
Where $K=7$ in both cases. On the other hand, the AM baseline is run for $1000$ gradient steps with learning rate of $1e^{-3}$ and $\alpha = 1.538$.

\paragraph{Conservative manifold exploration. } 
We ran \AlgNameShort for $K=50$ iterations  and $2500$ gradient steps in total with $\eta=10$ and $\alpha=0.0,0.01,0.1,0.5,1.0$. We set the AM learning rate to $2e^{-4}$ and sample trajectories of length $400$ for computing the AM loss. 
\begin{figure}[h]
\centering
\begin{tabular}{@{}lccc@{}}
          \toprule
          & $\EV[r(x)]$ & $W_1^A$ & $\Delta \%$ \\
          \midrule
          Pre & $29.5 \pm 0.0$ & $0$ & $-$ \\
          \AlgNameShortAM & $35.08 \pm 0.04$ & $4.68 \pm 0.0$ & $100$ \\
          \AlgNameShort-A & $35.38 \pm 0.04$ & $1.92 \pm 0.03$ & $288 \pm 15.0$ \\
          \bottomrule
\end{tabular}%
\caption{Statistical analysis for $W_1$ divergence.}
\end{figure}

\subsection{Further Ablations}
\paragraph{Runtime}
The only input hyperparameter in Algorithm \ref{alg:fdc_algorithm} is the number of iterations $K$. Towards evaluating its effect on algorithm execution, in the following, we consider an experimental setup analogous to "Risk-averse reward maximization for better worst-case validity or safety" experiment in Fig. \ref{fig:experiments_fig_1} (top row), and evaluate the effect of different numbers of iterations ($K$) on run-time and solution quality. We report results in Fig. \ref{fig:runtime_fig1} showing the depending on hyper-parameter $K$ for $\eta=20$. As one can expect, for small step-sizes $\frac{1}{\eta}$, the runtime and solution quality scale nearly linearly in $K$ given a fixed number of iterations $N$ of the entropy-regularized solver (see Apx. \ref{sec:alg_implementation} for the definition of $N$).
\begin{figure}[h]
\centering
\begin{tabular}{@{}lcc@{}}
\toprule
\textbf{K} & \textbf{Runtime (s)} & \textbf{CVaR estimate (via 1000 samples)} \\
\midrule
0 & 0.00   & 254.49 \\
1 & 44.71  & 241.37  \\
2 & 89.38  & 167.04  \\
3 & 133.31 & 288.72 \\
4 & 176.79 & 271.00 \\
5 & 220.37 & 84.89  \\
6 & 264.06 & 96.55  \\
\bottomrule
\end{tabular}
\caption{Runtime vs.\ CVaR estimate as a function of $K$, $\eta=20$.}
\label{fig:runtime_fig1}
\end{figure}
As one can expect by interpreting $\eta$ in Eq. \ref{eq:opt_first_variation} as a learning-rate parameter, by choosing smaller values of  $\eta$ convergence can be achieved with less iterations $K$. In Fig. \ref{fig:runtime_fig2}  we report the same evaluation with $\eta=10$.
\begin{figure}[h]
\centering
\begin{tabular}{@{}lcc@{}}
\toprule
\textbf{K} & \textbf{Runtime (s)} & \textbf{CVaR estimate (via 1000 samples)} \\
\midrule
0 & 0.00   & 254.49 \\
1 & 44.23  & 249.09  \\
2 & 87.78  & 90.0  \\
\bottomrule
\end{tabular}
\caption{Runtime vs.\ CVaR estimate as a function of $K$, $\eta=10$.}
\label{fig:runtime_fig2}
\end{figure}
The above tables hint at the fact that the \AlgNameShort fine-tuning process can be interpreted experimentally as classic (convex or non-convex) optimization, although on the space of generative models, with learning rate (or step-size) controlled by $\eta$.
\paragraph{Approximate Oracle}
In the following, we investigate the use of an approximate entropy-regularized control solver oracle (i.e., performing approximate entropy-regularized fine-tuning at each iteration of \AlgNameShort), showing that this can also lead to optimality via increasing the number of iterations $K$. In the following (see Fig. \ref{fig:approximate_oracle_exp}) we consider $N=100$ instead of $N=1000$ (as in previous experiments) and use $K=5$ showing that \AlgNameShort can retrieve the same final fine-tuned model as in Fig. \ref{fig:runtime_fig2} using only one tenth gradient steps (i.e.  $N=100$ instead of $N=1000$) for the inner oracle.
\begin{figure}[h]
\centering
\begin{tabular}{@{}lcc@{}}
\toprule
\textbf{K} & \textbf{Runtime (s)} & \textbf{CVaR estimate (via 1000 samples)} \\
\midrule
0 & 0.00   & 254.49 \\
1 & 5.11  & 221.1  \\
2 & 10.01  & 194.69  \\
3 & 14.73 & 90.01 \\
4 & 19.60 & 91.1 \\
5 & 24.60 & 90.0  \\
\bottomrule
\end{tabular}
\caption{Runtime vs.\ CVaR estimate as a function of $K$, $\eta=10$, $N=100$.}
\label{fig:approximate_oracle_exp}
\end{figure}

\subsection{Real-World Experiments}

\paragraph{Molecular design for single-point energy minimization. }
In this experiment \AlgNameShort is run for $K=10$ iterations, with merely $2$ gradient steps at each iteration (\ie the AM oracle is very approximate), AM learning rate of $1e^{-4}$, $\eta = 0.01$ and $\alpha = 0$. Meanwhile, the AM baseline is run for $240$ gradient steps with $\alpha = 0.0045$.  

\paragraph{Text-to-image bridge designs conservative exploration.} For this experiment we ran \AlgNameShort on a single Nvidia H100 GPU, with $K=2$, $\eta=200$, $\alpha=0.001$ and a $100$ gradient steps in total. Similarly to previous work, we tuned the vector field resulting from applying classifier-free guidance with guidance scale $w=8$ in SD-$1.5$.

\newpage

\end{document}